%% file: root.tex
\documentclass[journal]{IEEEtran}
\usepackage{etex}
\reserveinserts{100}
\usepackage{mathtools}
\usepackage{graphics, amsthm, times, amsfonts, graphicx, amssymb, cite}
\usepackage[outdir=./]{epstopdf}
\usepackage{tikz}
\usetikzlibrary{shapes,arrows}
\usepackage{pgfplots}
\usepackage{color}
\usepackage{setspace}
\usepackage{hyperref}
\usepackage{multirow}
\usepackage{rotating}
\usepackage{comment}
\usepackage{flushend}
\usepackage{caption}
\usepackage{subcaption}
\usepackage[affil-it]{authblk}
\usepackage{bm}
\usepackage{algorithm}
\usepackage{algorithmic}
\usepackage{enumerate}
\usepackage{enumitem}
\usepackage{morefloats}
\usepackage{setspace}

\input{mysymbol.sty}
\input{my_sections}

\newcommand{\QED}{\hfill\ensuremath{\blacksquare}}
\newcommand{\norm}[1]{\ensuremath{\left\| #1 \right\|}}

\def\E{\mathbb{E}}

\newtheorem{assumption}{\hspace{0pt}\bf Assumption}
\newtheorem{lemma}{\hspace{0pt}\bf Lemma}

\newtheorem{example}{\hspace{0pt}\bf Example}

\newtheorem{theorem}{\hspace{0pt}\bf Theorem}

\newtheorem{remark}{\hspace{0pt}\bf Remark}

\newtheorem{definition}{\hspace{0pt}\bf Definition}

\title{Learning Optimal Resource Allocations in \\ Wireless Systems}
\author{Mark Eisen$^{*}$ \quad Clark Zhang$^{*}$ \quad Luiz F. O. Chamon$^{*}$ \quad Daniel D. Lee$^{\dagger}$ \quad Alejandro Ribeiro$^{*}$
\thanks{{Supported by ARL DCIST CRA W911NF-17-2-0181 and Intel Science and Technology Center for Wireless Autonomous Systems. The authors are with the $^{*}$Department of Electrical and Systems Engineering, University of Pennsylvania and $^{\dagger}$Department of Electrical and Computer Engineering, Cornell Tech. Email: maeisen@seas.upenn.edu, clarkz@seas.upenn.edu, luizf@seas.upenn.edu, ddl46@cornell.edu, aribeiro@seas.upenn.edu}.}
}

\def \fph  {\bbf\big(\bbp(\bbh), \bbh\big)}
\def \Efph {\mbE\big[\fph\big]}
\def \fphzero  {\bbf\big(\bbp_0(\bbh),\bbh\big)}
\def \Efphzero {\mbE\big[\fphzero\big]}
\def \fphih  {\bbf\big(\bbphi(\bbh,\bbtheta),\bbh\big)}
\def \Efphih {\mbE\big[\fphih\big]}
\def \fpihk {\bbf\big(\bbh_k,\bbphi(\bbh_k,\bbtheta_k)\big)}
\def \Lagphi {\ccalL_{\bbphi}(\bbtheta,\bbx, \bblambda, \bbmu)}
\begin{document}

\thispagestyle{empty}
\maketitle

\begin{abstract}
This paper considers the design of optimal resource allocation policies in wireless communication systems which are generically modeled as a functional optimization problem with stochastic constraints. These optimization problems have the structure of a learning problem in which the statistical loss appears as a constraint, motivating the development of learning methodologies to attempt their solution. To handle stochastic constraints, training is undertaken in the dual domain. It is shown that this can be done with small loss of optimality when using near-universal learning parameterizations. In particular, since deep neural networks (DNN) are near-universal their use is advocated and explored. DNNs are trained here with a model-free primal-dual method that simultaneously learns a DNN parametrization of the resource allocation policy and optimizes the primal and dual variables. Numerical simulations demonstrate the strong performance of the proposed approach on a number of common wireless resource allocation problems.
\end{abstract}


%%%%%%%%%%%%%%%%%%%%%%%%%
%%%%%%% S E C T I O N %%%%%%%%%%%
%%%%%%%%%%%%%%%%%%%%%%%%%%
%%%%%%%%%%%%%%%%%%%%%%%%%%%%%%%%%%%%%%%%%%%%%%%%%%%%%%%%%%%%%%%%%%%%%

\section{Introduction}
\input{introduction.tex}

%%%%%%%%%%%%%%%%%%%%%%%%%%%%%%%%%%%%%%%%%%%%%%%%%%%%%%%%%%%%%%%%%
%%%%%%%%%%%%%%%%%%%%%%%%%
%%%%%%% S E C T I O N %%%%%%%%%%%
%%%%%%%%%%%%%%%%%%%%%%%%%%
%%%%%%%%%%%%%%%%%%%%%%%%%%%%%%%%%%%%%%%%%%%%%%%%%%%%%%%%%%%%%%%%%%%%%

\section{Optimal Resource Allocation in Wireless Communication Systems}\label{sec_problem}

\input{problem_formulation.tex}

%%%%%%%%%%%%%%%%%%%%%%%%%%%%%%%%%%%%%%%%%%%%%%%%%%%%%%%%%%%%%%%%%
%%%%%%%%%%%%%%%%%%%%%%%%%
%%%%%%% S E C T I O N %%%%%%%%%%%
%%%%%%%%%%%%%%%%%%%%%%%%%%
%%%%%%%%%%%%%%%%%%%%%%%%%%%%%%%%%%%%%%%%%%%%%%%%%%%%%%%%%%%%%%%%%%%%%
\section{Lagrangian Dual Problem}\label{sec_dual}

\input{duality.tex}

%%%%%%%%%%%%%%%%%%%%%%%%%%%%%%%%%%%%%%%%%%%%%%%%%%%%%%%%%%%%%%%%%
%%%%%%%%%%%%%%%%%%%%%%%%%
%%%%%%% S E C T I O N %%%%%%%%%%%
%%%%%%%%%%%%%%%%%%%%%%%%%%
%%%%%%%%%%%%%%%%%%%%%%%%%%%%%%%%%%%%%%%%%%%%%%%%%%%%%%%%%%%%%%%%%%%%%
\section{Model-Free Learning}\label{sec_model_free}
 
\input{model_free.tex}

%%%%%%%%%%%%%%%%%%%%%%%%%%%%%%%%%%%%%%%%%%%%%%%%%%%%%%%%%%%%%%%%%
%%%%%%%%%%%%%%%%%%%%%%%%%
%%%%%%% S E C T I O N %%%%%%%%%%%
%%%%%%%%%%%%%%%%%%%%%%%%%%
%%%%%%%%%%%%%%%%%%%%%%%%%%%%%%%%%%%%%%%%%%%%%%%%%%%%%%%%%%%%%%%%%%%%%
\section{Deep Neural Networks}\label{sec_dnn}

\input{neural_networks.tex}

%%%%%%%%%%%%%%%%%%%%%%%%%%%%%%%%%%%%%%%%%%%%%%%%%%%%%%%%%%%%%%%%%%
%%%%%%%%%%%%%%%%%%%%%%%%%%
%%%%%%%% S E C T I O N %%%%%%%%%%%
%%%%%%%%%%%%%%%%%%%%%%%%%%%
%%%%%%%%%%%%%%%%%%%%%%%%%%%%%%%%%%%%%%%%%%%%%%%%%%%%%%%%%%%%%%%%%%%%%%
%\section{Theoretical Analysis}\label{sec_analysis}
%
%\input{analysis.tex}

%%%%%%%%%%%%%%%%%%%%%%%%%%%%%%%%%%%%%%%%%%%%%%%%%%%%%%%%%%%%%%%%%
%%%%%%%%%%%%%%%%%%%%%%%%%
%%%%%%% S E C T I O N %%%%%%%%%%%
%%%%%%%%%%%%%%%%%%%%%%%%%%
%%%%%%%%%%%%%%%%%%%%%%%%%%%%%%%%%%%%%%%%%%%%%%%%%%%%%%%%%%%%%%%%%%%%%
\section{Simulation Results}\label{sec_simulation_results}

\input{numerical_results.tex}

%%%%%%%%%%%%%%%%%%%%%%%%%%%%%%%%%%%%%%%%%%%%%%%%%%%%%%%%%%%%%%%%%
%%%%%%%%%%%%%%%%%%%%%%%%%
%%%%%%% S E C T I O N %%%%%%%%%%%
%%%%%%%%%%%%%%%%%%%%%%%%%%
%%%%%%%%%%%%%%%%%%%%%%%%%%%%%%%%%%%%%%%%%%%%%%%%%%%%%%%%%%%%%%%%%%%%%
\section{Conclusion}
In this paper, we studied a generic formulation for resource allocation in wireless systems. The functional-optimization, non-convex constraints, and lack of model knowledge makes these problems challenging, if not impossible, to solve directly. We used the concept of universal function approximation of deep neural networks and the theory of Lagrangian duality to show that, despite the non-convex nature of these problems, they can be formulated with a finite-dimensional, unconstrained optimization problem in the dual domain with either bounded suboptimality, or in the case of arbitrarily large DNNs, precise optimality with respect to the original problem. The dual domain formulation motivates solving via the use of primal-dual descent methods, which can furthermore be replaced with zeroth-ordered equivalents that estimate gradients without explicit model knowledge. We additionally perform a variety of simulations on common resource allocation problems that demonstrate the effectiveness in DNN-parameterizations to find accurate solutions.
%%%%%%%%%%%%%%%%%%%%%%%%%%%%%%%%%%%%%%%%%%%%%%%%%%%%%%%%%%%%%%%%%
%%%%%%%%%%%%%%%%%%%%%%%%%
%%%%%%% A P P E N D I X     A    %%%%%%%%%%%
%%%%%%%%%%%%%%%%%%%%%%%%%%
%%%%%%%%%%%%%%%%%%%%%%%%%%%%%%%%%%%%%%%%%%%%%%%%%%%%%%%%%%%%%%%%%%%%%
\appendices

\section{Proof of Theorem \ref{theorem_param_duality}}\label{app_param_duality}

\input{duality_proof.tex}

\urlstyle{same}
\bibliographystyle{IEEEtran}
\bibliography{wireless_learning}

\end{document}

%% file: my_sections.tex
\usepackage{needspace}

%% file: introduction.tex
The defining feature of wireless communication is fading and the role of optimal wireless system design is to allocate resources across fading states to optimize long term system properties. Mathematically, we have a random variable $\bbh$ that represents the instantaneous fading environment, a corresponding instantaneous allocation of resources $\bbp(\bbh)$, and an instantaneous performance outcome $\fph$ resulting from the allocation of resources $\bbp(\bbh)$ when the channel realization is $\bbh$. The instantaneous system performance tends to vary too rapidly from the perspective of end users for whom the long term average $\bbx = \Efph$ is a more meaningful metric. This interplay between instantaneous allocation of resources and long term performance results in distinctive formulations where we seek to maximize a utility of the long term average $\bbx$ subject to the constraint $\bbx = \Efph$. Problems of this form range from the simple power allocation in wireless fading channels -- the solution of which is given by water filling -- to the optimization of frequency division multiplexing \cite{wang2011resource}, beamforming \cite{sidiropoulos2006transmit, bazerque2008distributed}, and random access \cite{hu2011adaptive,hu2012optimal}.

Optimal resource allocation problems are as widespread as they are challenging. This is because of the high dimensionality that stems from the variable $\bbp(\bbh)$ being a function over a dense set of fading channel realizations and the lack of convexity of the constraint $\bbx = \Efph$. For resource allocation problems, such as interference management, heuristic methods have been developed \cite{shi2011iteratively, chen2011round, wu2013flashlinq}. Generic solution methods are often undertaken in the Lagrangian dual domain. This is motivated by the fact that the dual problem is not functional, as it has as many variables as constraints, and is always convex whether the original problem is convex or not. A key property that enables this solution is the lack of duality gap, which allows dual operation without loss of optimality. The duality gap has long being known to be null for convex problems -- e.g., the water level in water filling solutions is a dual variable -- and has more recently being shown to be null under mild technical conditions despite the presence of the nonconvex constraint~$\bbx = \Efph$~\cite{yu2006dual, ribeiro2012optimal}. This permits dual domain operation in a wide class of problems and has lead to formulations that yield problems that are {\it more} tractable, although not necessarily tractable without resorting to heuristics~\cite{zhang2006stochastic, GatsisEtal15,wang2016dynamic, eisen2018learning, liu2001opportunistic, eryilmaz2007fair, ntranos2009multicast}. 

The inherent difficulty of resource allocation problems makes the use of machine learning tools appealing. One may collect a training set composed of optimal resource allocations $\bbp^*(\bbh_k)$ for some particular instances $\bbh_k$ and utilize the learning parametrization to interpolate solutions for generic instances $\bbh$. The bottleneck step in this learning approach is the acquisition of the training set. In some cases this set is available by reverse engineering as it is possible to construct a problem having a given solution \cite{farsad2017detection, sun2018limited}. In some other cases heuristics can be used to find approximate solutions to construct a training set \cite{sun2017learning, lei2017deep, lee2018deep}. This limits the performance of the learning solution to the performance of the heuristic, though the methodology has proven to work well at least in some particular problems. 

Instead of acquiring a training set, one could exploit the fact that the expectation $\Efph$ has a form that is typical of learning problems. Indeed, in the context of learning, $\bbh$ represents a feature vector, $\bbp(\bbh)$ the regression function to be learned, $\fph$ a loss function to be minimized, and the expectation $\Efph$ the statistical loss over the distribution of the dataset. We may then learn without labeled training data by directly minimizing the statistical loss with stochastic optimization methods which merely observe the loss $\fph$ at sampled pairs $(\bbh,\bbp(\bbh))$. This setting is typical of, e.g., reinforcement learning problems \cite{sutton1998reinforcement}, and is a learning approach that has been taken in several {\it unconstrained} problems in wireless optimization \cite{de2017decentralized, o2017physical, o2017introduction, ye2018deep}. In general, wireless optimization problems {\it do} have constraints as we are invariably trying to balance capacity, power consumption, channel access, and interference. Still, the fact remains that wireless optimization problems have a structure that is inherently similar to learning problems. This realization is the first contribution of this paper:

\bi [(C1)] Parametrizing the resource allocation function $\bbp(\bbh)$ yields an optimization problem with the structure of a learning problem in which the statistical loss appears as a constraint (Section \ref{sec_problem}). \ei

\noindent This observation is distinct from existing work in learning for wireless resource allocation. Whereby existing works apply machine learning methods to wireless resource allocation, such as via supervised training, here we identify that the wireless resource allocation is \emph{itself} a statistical learning problem. This motivates the use of learning methods to directly solve the resulting optimization problems bypassing the acquisition of a training set. To do so, it is natural to operate in the dual domain where constraints are linearly combined to create a weighted objective (Section \ref{sec_dual}). The first important question that arises in this context is the price we pay for learning in the dual domain. Our second contribution is to show that this question depends on the quality of the learning parametrization. In particular, if we use learning representations that are near universal---meaning that they can approximate any function up to a specified accuracy (Definition \ref{def_universal})----we can show that dual training is close to optimal:

\bi [(C2)]The duality gap of learning problems in wireless optimization is small if the learning parametrization is nearly universal (Section \ref{sec_analysis}). More formally, the duality gap is $\ccalO(\eps)$ if the learning parametrization can approximate arbitrary functions with error~$\ccalO(\eps)$~(Theorem \ref{theorem_param_duality}).\ei 

\noindent A second question that we address is the design of training algorithms for optimal resource allocation in wireless systems. The reformulation in the dual domain gives natural rise to a gradient-based, primal-dual learning method (Section \ref{sec_primal_dual}). The primal-dual method cannot be implemented directly, however, because computing gradients requires unavailable model knowledge. This motivates a third contribution:

%\red{\bi [(C3)] We propose a stochastic primal-dual method to train learning parameterizations in optimal wireless system design (Section \ref{sec_primal_dual}).\ei} \blue{Mark: I don't feel this is a contribution. The algorithm we propose is the model-free version in C5}

\bi [(C3)] We introduce a model-free learning approach, in which gradients are estimated by sampling the model functions and wireless channel (Section \ref{sec_model_free}).\ei

\noindent This model-free approach additionally includes the policy gradient method for efficiently estimating the gradients of a function of a policy (Section \ref{sec_policy_grad}). We remark that since the optimization problem is not convex, the primal-dual method does not converge to the optimal solution of the learning problem but to a stationary point of the KKT conditions \cite{boyd2004convex}. This is analogous to unconstrained learning where stochastic gradient descent is known to converge only to a local minima. 

The quality of the learned solution inherently depends on the ability of the learning parametrization to approximate the optimal resource allocation function. In this paper we advocate for the use of neural networks:

\bi [(C4)] We consider the use of deep neural networks (DNN) and conclude that since they are universal parameterizations, they can be trained in the dual domain without loss of optimality~(Section \ref{sec_dnn}).\ei

\noindent Together, the Lagrangian dual formulation, model-free algorithm, and DNN parameterization provide a practical means of learning in resource allocation problems with near-optimality. We conclude with a series of simulation experiments on a set of common wireless resource allocation problems, in which we demonstrate the near-optimal performance of the proposed DNN learning approach (Section \ref{sec_simulation_results}).

%% file: problem_formulation.tex
Let $\bbh \in \ccalH \subseteq \reals_+^n$ be a random vector representing a collection of~$n$ stationary wireless fading channels drawn according to the probability distribution~$m(\bbh)$. Associated with each fading channel realization, we have a resource allocation vector $\bbp(\bbh)\in\reals^m$ and a function $\bbf:\reals^m\times\reals^n\to\reals^u$. The components of the vector valued function $\fph$ represent performance metrics that are associated with the allocation of resources $\bbp(\bbh)$ when the channel realization is $\bbh$. In fast time varying fading channels, the system allocates resources instantaneously but users get to experience the average performance across fading channel realizations. This motivates considering the vector ergodic average~$\bbx = \Efph \in\reals^u$, which, for formulating optimal wireless design problems, is relaxed to the inequality
\begin{align}\label{eqn_ergo_avg}
   \bbx \leq \Efph.
\end{align}
In \eqref{eqn_ergo_avg}, we interpret $\Efph$ as the level of service that is available to users and $\bbx$ as the level of service utilized by users. In general we will have  $\bbx = \Efph$ at optimal operating points, but this is not required a priori.

The goal in optimally designed wireless communication systems is to find the instantaneous resource allocation $\bbp(\bbh)$ that optimizes the performance metric $\bbx$ in some sense. To formulate this problem mathematically we introduce a vector utility function $\bbg:\reals^u\to\reals^r$ and a scalar utility function $g_0:\reals^u\to\reals$, taking values $\bbg(\bbx)$ and $g_0(\bbx)$, that measure the value of the ergodic average $\bbx$. We further introduce the set $\ccalX\subseteq\reals^u$ and $\ccalP \subseteq \ccalM$, where~$\ccalM$ is the set of functions integrable with respect to~$m(\bbh)$, to constrain the values that can be taken by the ergodic average and the instantaneous resource allocation, respectively. We assume~$\ccalP$ contains bounded functions, i.e., that the resources being allocated are finite. With these definitions, we let the optimal resource allocation problem in wireless communication systems be a program of the form
\begin{alignat}{3} \label{eq_problem}
   P^* := &  \max_{\bbp(\bbh),\bbx} \ && g_0(\bbx),             \nonumber \\
        &  \st                    \ && \bbx       \leq    \Efph,   \nonumber \\
        &                         \ && \bbg(\bbx) \geq \bbzero, \ 
                                       \bbx       \in  \ccalX,  \
                                       \bbp \in  \ccalP.  %   
\end{alignat}
In \eqref{eq_problem} the utility $g_0 (\bbx)$ is the one we seek to maximize while the utilities $\bbg(\bbx)$ are required to be nonnegative. The constraint $\bbx \leq \Efph$ relates the instantaneous resource allocations with the long term average performances as per \eqref{eqn_ergo_avg}. The constraints $\bbx \in \ccalX$ and $\bbp \in \ccalP$ are set restrictions on $\bbx$ and $\bbp$. The utilities $g_0 (\bbx)$ and $\bbg(\bbx)$ are assumed to be concave and the set $\ccalX$ is assumed to be convex. However, the function~$\bbf(\cdot,\bbh)$ is not assumed convex or concave and the set $\ccalP$ is not assumed to be convex either. In fact, the realities of wireless systems make it so that they are typically non-convex~\cite{ribeiro2012optimal}. We present three examples below to clarify ideas and proceed to motivate and formulate learning approaches for solving \eqref{eq_problem}. 

%%%%%%%%%%%%%%%%%%%%%%%%%%%%%%%%%%%%%%%%%%%%%%%%%%%%%%%%%%%%%%%%%%%%%%%%%%%%%%%%%%%%%%%%%
%%%%   E   X   A   M   P   L   E   %%%%%%%%%%%%%%%%%%%%%%%%%%%%%%%%%%%%%%%%%%%%%%%%%%%%%%
%%%%%%%%%%%%%%%%%%%%%%%%%%%%%%%%%%%%%%%%%%%%%%%%%%%%%%%%%%%%%%%%%%%%%%%%%%%%%%%%%%%%%%%%%
%
\begin{example}[Point-to-point wireless channel]\label{example_point_to_point} \normalfont 
In a point-to-point channel we measure the channel state $h$ and allocate power $p(h)$ to realize a rate $c(p(h);h) = \log(1+hp(h))$ assuming the use of capacity achieving codes. The metrics of interest are the average rate $c = \mbE_h[c(p(h);h)] = \mbE_h [\log(1+hp(h))]$ and the average power consumption $p = \mbE_h[p(h)]$. These two constraints are of the ergodic form in \eqref{eqn_ergo_avg}. We can formulate a rate maximization problem subject to power constraints with the utility $g_0(\bbx) = g_0 (c,p) = c$ and the set $\ccalX = \{p: 0 \leq p \leq p_0\}$. Observe that the utility is concave (linear) and the set $\ccalX$ is convex (a segment). In this particular case the instantaneous performance functions $\log(1+hp(h))$ and $p(h)$ are concave. A similar example in which the instantaneous performance functions are not concave is when we use a set of adaptive modulation and coding modes. In this case the rate function $c(p(h);h)$ is a step function \cite{ribeiro2012optimal}. \end{example}

%%%%%%%%%%%%%%%%%%%%%%%%%%%%%%%%%%%%%%%%%%%%%%%%%%%%%%%%%%%%%%%%%%%%%%%%%%%%%%%%%%%%%%%%
%%%   E   X   A   M   P   L   E   %%%%%%%%%%%%%%%%%%%%%%%%%%%%%%%%%%%%%%%%%%%%%%%%%%%%%%
%%%%%%%%%%%%%%%%%%%%%%%%%%%%%%%%%%%%%%%%%%%%%%%%%%%%%%%%%%%%%%%%%%%%%%%%%%%%%%%%%%%%%%%%
%
\begin{example}[Multiple access interference channel]\label{example_interference} \normalfont
A set of $m$ terminals communicates with associated receivers. The channel linking terminal $i$ to the its receiver is $h^{ii}$ and the interference channel to receiver $j$ is given by $h^{ji}$. The power allocated in this channel is $p^i(\bbh)$ where $\bbh=[h^{11}; h^{12};\ldots;h^{mm}]$. The instantaneous rate achievable by terminal $i$ depends on the signal to interference plus noise ratio (SINR) $c^i(\bbp(\bbh);\bbh) = h^{ii} p^i (\bbh) /[ 1 + \sum_{j\neq i} h^{ji} p^j (\bbh) ]$. Again, the quantity of interest for each terminal is the long term rate which, assuming use of capacity achieving codes, is
\begin{align}\label{eq_interference_problem1}
  x^i \leq \mbE_{\bbh} \bigg[ \log \bigg(1 + \frac{h^{ii} p^i(\bbh)}
              {1 + \sum_{j \neq i} h^{ji} p^j(\bbh)}\bigg)\bigg].
\end{align}
The constraint in \eqref{eq_interference_problem1} has the form of \eqref{eqn_ergo_avg} as it relates instantaneous rates with long term rates. The problem formulation is completed with a set of average power constraints $p^i = \mbE_\bbh[p^i(\bbh)]$. Power constraints can be enforced via the set $\ccalX = \{p: 0 \leq p \leq p_0\}$ and the utility $g_0$ can be chosen to be the weighted sum rate $g_0(\bbx) = \sum_{i} w^i x^i$ or a proportional fair utility  $g_0(\bbx) = \sum_{i} \log(x^i)$. Observe that the utility is concave but the instantaneous rate function $c^i(\bbp(\bbh);\bbh)$ is not convex. A twist on this problem formulation is to make $\ccalP=\{0,1\}^m$ in which case individual terminals are either active or not for a given channel realization. Although this set $\ccalP$ is not convex, it is allowed in \eqref{eq_problem}. \end{example}

%%%%%%%%%%%%%%%%%%%%%%%%%%%%%%%%%%%%%%%%%%%%%%%%%%%%%%%%%%%%%%%%%%%%%%%%%%%%%%%%%%%%%%%%
%%%   E   X   A   M   P   L   E   %%%%%%%%%%%%%%%%%%%%%%%%%%%%%%%%%%%%%%%%%%%%%%%%%%%%%%
%%%%%%%%%%%%%%%%%%%%%%%%%%%%%%%%%%%%%%%%%%%%%%%%%%%%%%%%%%%%%%%%%%%%%%%%%%%%%%%%%%%%%%%%
%
\begin{example}[Time division multiple access]\label{example_fdma} \normalfont
In Example \ref{example_interference} terminals are allowed to transmit simultaneously. Alternatively, we can request that only one terminal be active at any point in time. This can be modeled by introducing the scheduling variable $\alpha^i(\bbh)\in\{0,1\}$ and rewriting the rate expression in~\eqref{eq_interference_problem1} as
\begin{align}\label{eq_tdma_problem1}
  x^i \leq \mbE_{\bbh} \Big[ \alpha^i(\bbh) \log \big(1 + h^{i} p^i(\bbh) \big)\Big],
\end{align}
where the interference term does not appear because we restrict channel occupancy to a single terminal. To enforce this constraint we define the set $\ccalP:=\{\alpha^i(\bbh):\alpha^i(\bbh)\in\{0,1\}, \sum_{i}\alpha^i(\bbh)\leq 1\}$. This is a problem formulation in which, different from Example  \ref{example_interference}, we not only allocate power but channel access as well. \end{example}

%%%%%%%%%%%%%%%%%%%%%%%%%%%%%%%%%%%%%%%%%%%%%%%%%%%%%%%%%%%%%%%%%%%%%%%%%%%%%%%%%%%%%%%%
%%%   S   E   C   T   I   O   N   %%%%%%%%%%%%%%%%%%%%%%%%%%%%%%%%%%%%%%%%%%%%%%%%%%%%%%
%%%%%%%%%%%%%%%%%%%%%%%%%%%%%%%%%%%%%%%%%%%%%%%%%%%%%%%%%%%%%%%%%%%%%%%%%%%%%%%%%%%%%%%%
%
\subsection{Learning formulations}

The problem in \eqref{eq_problem}, which formally characterizes the optimal resource allocation policies for a diverse set of wireless problems, is generally a very difficult optimization problem to solve. In particular, two well known challenges in solving \eqref{eq_problem} directly are: 

   \bi[{(i)}]   The optimization variable $\bbp$ is a function. 
   \i [{(ii)}]  The channel distribution $m(\bbh)$ is unknown. \ei

\noindent Challenge (ii) is of little concern as it can be addressed with stochastic optimization algorithms. Challenge (i) makes \eqref{eq_problem} a functional optimization problem, which, compounded with the fact that \eqref{eqn_ergo_avg} defines a nonconvex constraint, entails large computational complexity. This is true even if we settle for a local minimum because we need to sample the $n$-dimensional space $\ccalH$ of fading realizations $\bbh$. If each channel is discretized to $d$ values the number of resource allocation variables to be determined is $m d^n$. As it is germane to the ideas presented in this paper, we point that \eqref{eq_problem} is known to have null duality gap \cite{ribeiro2012optimal}. This, however, does not generally make the problem easy to solve and moreover requires having model information.

This brings a third challenge in solving \eqref{eq_problem}, namely the availability of the wireless system functions:

   \bi [{(iii)}] The form of the instantaneous 
                performance function $\fph$, utility $g_0(\bbx)$, and constraint $\bbg(\bbx)$ may not be known. \ei

\noindent As we have seen in Examples \ref{example_point_to_point}-\ref{example_fdma}, the function $\fph$ models instantaneous achievable rates. Although these functions {\it may} be available in ideal settings, there are difficulties in measuring the radio environment that make them uncertain. This issue is often neglected but it can cause significant discrepancies between predicted and realized performances. Moreover, with less idealized channel models or performance rate functions---such as bit error rate---reliable models may even not be available to begin with. While the functions $g_0(\bbx)$ and $\bbg(\bbx)$ are sometimes known or designed by the user, we assume they are not here for complete generality.

Challenges (i)-(iii) can all be overcome with the use of a learning formulation. This is accomplished by introducing a parametrization of the resource allocation function so that for some $\bbtheta \in \reals^q$ we make 
\begin{align}\label{eqn_learning_param}
   \bbp(\bbh) = \bbphi(\bbh,\bbtheta).
\end{align}
With this parametrization the ergodic constraint in \eqref{eqn_ergo_avg} becomes
\begin{align}\label{eqn_ergo_avg_learning}
   \bbx \leq \Efphih
\end{align}
If we now define the set $\Theta := \{ \bbtheta \mid \bbphi(\bbh,\bbtheta) \in \ccalP\}$, the optimization problem in \eqref{eq_problem} becomes one in which the optimization is over $\bbx$ and $\bbtheta$
\begin{alignat}{3} \label{eq_param_problem}
   P^*_{\bbphi} := &  \max_{\bbtheta,\bbx}   \ && g_0 (\bbx),             \nonumber \\
        &  \st                    \ &&  \bbx \leq  \Efphih, \nonumber \\
        &                         \ && \bbg(\bbx) \geq \bbzero, \ 
                                       \bbx       \in  \ccalX,  \
                                       \bbtheta \in \Theta.  %   
\end{alignat}
Since the optimization is now carried over the parameter $\bbtheta\in\reals^q$ and the ergodic variable $\bbx\in\reals^u$, the number of variables in \eqref{eq_param_problem} is $q + u$. This comes at a loss of of optimality because \eqref{eqn_learning_param} restricts resource allocation functions to adhere to the parametrization $\bbp(\bbh) = \bbphi(\bbh,\bbtheta)$. E.g., if we use a linear parametrization $\bbp(\bbh)=\bbtheta^T\bbh$ it is unlikely that the solutions of \eqref{eq_problem} and \eqref{eq_param_problem} are close. In this work, we focus our attention on a widely-used class of parameterizations we define as \emph{near-universal}, which are able to model any function in $\ccalP$ to within a stated accuracy. We present this formally in the following definition.
\begin{definition}\label{def_universal}
A parameterization $\bbphi(\bbh,\bbtheta)$ is an $\epsilon$-universal parameterization of functions in $\ccalP$ if, for some $\eps > 0$, there exists for any $\bbp \in \ccalP$ a parameter $\bbtheta \in \Theta$ such that
\begin{equation}\label{eq_def_bound}
	 \E \left\| \bbp(\bbh) - \bbphi(\bbh,\bbtheta) \right\|_{\infty}
		\leq \epsilon.
\end{equation}
\end{definition}
A number of popular machine learning models are known to exhibit the universality property in Definition \ref{def_universal}, such as radial basis function networks (RBFNs) \cite{park1991universal} and reproducing kernel Hilbert spaces (RKHS) \cite{sriperumbudur2010relation}. This work focuses in particular on deep neural networks (DNNs), which can be shown to exhibit a universal function approximation property \cite{hornik1991approximation} and are observed to work remarkably well in practical problems---see, e.g, \cite{krizhevsky2012imagenet,long2015fully}. The specific details regarding the use of DNNs in the proposed learning framework of this paper are discussed in Section \ref{sec_dnn}.

While the reduction of the dimensionality of the optimization space is valuable, the most important advantage of \eqref{eq_param_problem} is that we can use training to bypass the need to estimate the distribution $m(\bbh)$ and the functions $\fph$. The idea is to learn over a time index $k$ across observed channel realizations $\bbh_k$ and probe the channel with tentative resource allocations $\bbp_k(\bbh_k) = \bbphi(\bbh_k,\bbtheta_k)$. The resulting performance $\fpihk$ is then observed and utilized to learn the optimal parametrized resource allocation as defined by \eqref{eq_param_problem}. The major challenge to realize this idea is that existing learning methods operate in unconstrained optimization problems. We will overcome this limitation by operating in the dual domain where the problem is unconstrained (Section \ref{sec_dual}). Our main result on learning for constrained optimization is to show that, its lack of convexity notwithstanding, the duality gap of \eqref{eq_param_problem} is small for near-universal parameterizations~(Theorem \ref{theorem_param_duality}). This result justifies operating in the dual domain as it does not entail a significant loss of optimality. A model-free primal-dual method to train \eqref{eq_param_problem} is then introduced in Section \ref{sec_model_free} and neural network parameterizations are described in Section \ref{sec_dnn}.

%\medskip\noindent{\bf Notation remark. } Throughout the paper we use non-boldface variables, e.g. $c$, to reflect scalar values and boldface variables, e.g. $\bby$, to reflect vector values. Accordingly, non-boldface functions, e.g. $C(\bby)$, are scalar-valued and boldface functions, e.g. $\bbf(\bby)$, are vector-valued. The $\nabla$ symbol is then given to denote either a gradient, e.g. $\nabla C(\bby)$, or a Jacobian matrix, e.g. $\nabla \bbf(\bby)$, when applied to scalar and vector-valued functions, respectively. We use the notation $ \mathsf{P}_{\ccalA} (\tby) := \argmin_{\bby \in \ccalA} \| \bby-\tby\|^2$ to signify the projection of vector $\tby$ onto a set $\ccalA$, and additionally employ the simplified notation $[\tby]_+$ to signify the projection onto the positive orthant $\reals_+$.

%% file: duality.tex
%!TEX root = root.tex

Solving the optimization problem in \eqref{eq_param_problem} requires learning both the parameter $\bbtheta$ and the ergodic average variables $\bbx$ over a set of both convex and non-convex constraints. This can be done by formulating and solving the Lagrangian dual problem. To do so, introduce the nonnegative multiplier dual variables $\bblambda \in \reals_+^p$ and $\bbmu \in \reals_+^r$, respectively associated with the constraints $\bbx\leq\Efphih$ and $\bbg(\bbx)\leq\bbzero$. The Lagrangian of \eqref{eq_param_problem} is an average of objective and constraint values weighted by their respective multipliers:
\begin{align}\label{eq_param_lagrangian}
   \Lagphi &:=   g_0(\bbx) 
                + \bbmu^T \bbg(\bbx)
                \\
                {}&+ \bblambda^T \Big( \Efphih - \bbx \Big). \nonumber
\end{align}
With the Lagrangian so defined, we introduce the dual function $D_{\bbphi}(\bblam,\bbmu)$ as the maximum Lagrangian value attained over all $\bbx\in\ccalX$ and $\bbtheta \in \Theta$
\begin{align}\label{eq_param_dual_function}   
   D_{\bbphi}(\bblam,\bbmu) 
         := \max_{\bbtheta \in \Theta, \bbx \in \ccalX}  \Lagphi .
\end{align}
We think of \eqref{eq_param_dual_function} as a penalized version of \eqref{eq_param_problem} in which the constraints are not enforced but their violation is penalized by the Lagrangian terms $\bbmu^T \bbg(\bbx)$ and $\bblambda^T (\Efphih - \bbx)$. This interpretation is important here because the problem in~\eqref{eq_param_dual_function} is unconstrained except for the set restrictions $\bbtheta \in \Theta$ and $\bbx \in \ccalX$. This renders \eqref{eq_param_dual_function} analogous to conventional learning objectives and, as such, a problem that we can solve with conventional learning algorithms.

It is easy to verify and well-known that for any choice of $\bblambda \geq\bbzero$ and $\bbmu \geq \bbzero$ we have $D_{\bbphi}(\bblam,\bbmu)\geq P_{\bbphi}^*$. This motivates definition of the dual problem in which we search for the multipliers that make $D_{\bbphi}(\bblam,\bbmu)$ as small as possible
\begin{align}\label{eq_param_dual0}
   D_{\bbphi}^* &:= \min_{\bblambda,\bbmu \geq \bb0} D_{\bbphi}(\bblam,\bbmu).
\end{align}
The dual optimum $D_{\bbphi}^*$ is the best approximation we can have of $P_{\bbphi}^*$ when using \eqref{eq_param_dual_function} as a proxy for \eqref{eq_param_problem}. It follows that the two concerns that are relevant in utilizing \eqref{eq_param_dual_function} as a proxy for~\eqref{eq_param_problem} are: (i)~evaluating the difference between $D_{\bbphi}^*$ and $P_{\bbphi}^*$ and (ii)~designing a method for finding the optimal multipliers that attains the minimum in \eqref{eq_param_dual0}. We address (i) in Section \ref{sec_analysis} and (ii) in Section \ref{sec_primal_dual}.

%%%%%%%%%%%%%%%%%%%%%%%%%%%%%%%%%%%%%%%%%%%%%%%%%%%%%%%%%%%%%%%%%%%%%%%%%%%%%%%%%%%%%%%%
%%%   S   E   C   T   I   O   N   %%%%%%%%%%%%%%%%%%%%%%%%%%%%%%%%%%%%%%%%%%%%%%%%%%%%%%
%%%%%%%%%%%%%%%%%%%%%%%%%%%%%%%%%%%%%%%%%%%%%%%%%%%%%%%%%%%%%%%%%%%%%%%%%%%%%%%%%%%%%%%%
%

\subsection{Suboptimality of the dual problem}\label{sec_analysis}

\input{analysis.tex}

\subsection{Primal-Dual learning}\label{sec_primal_dual}

In order to train the parametrization $\bbphi(\bbh,\bbtheta)$ on the problem \eqref{eq_param_problem} we propose a \emph{primal-dual} optimization method. A primal-dual method performs gradient updates directly on both the primal and dual variables of the Lagrangian function in \eqref{eq_param_lagrangian} to find a local stationary point of the KKT conditions of \eqref{eq_param_problem}. In particular, consider that we successively update both the primal variables $\bbtheta, \bbx$ and dual variables $\bblambda, \bbmu$ over an iteration index $k$. At each index $k$ of the primal-dual method, we update the current primal iterates $\bbtheta_k, \bbx_k$ by adding the corresponding partial gradients of the Lagrangian in \eqref{eq_param_lagrangian}, i.e. $\nabla_{\bbtheta} \ccalL, \nabla_{\bbx} \ccalL$, and projecting to the corresponding feasible set, i.e.,
\begin{align}
\bbtheta_{k+1} &= \mathsf{P}_{\Theta}\left[\bbtheta_k + \gamma_{\bbtheta,k} \nabla_{\bbtheta}
\E \bbf(\bbphi(\bbh,\bbtheta_k),\bbh)\bblambda_k  \right], \label{eq_pd_update1} \\
\bbx_{k+1} &= \mathsf{P}_{\ccalX}\left[\bbx_k  +  \gamma_{\bbx,k} (\nabla g_0(\bbx) + \nabla \bbg(\bbx_{k})\bbmu_k - \bblambda_k) \right] , \label{eq_pd_update2}
\end{align}
where we introduce $\gamma_{\bbtheta,k} ,\gamma_{\bbx,k}>0$ as scalar step sizes. Likewise, we perform a gradient update on current dual iterates $\bblambda_k, \bbmu_k$ in a similar manner---by \emph{subtracting} the partial stochastic gradients  $\nabla_{\bblambda} \ccalL, \nabla_{\bbmu} \ccalL$ and projecting onto the positive orthant to obtain
\begin{align}
\bblambda_{k+1} &= \left[ \bblambda_k - \gamma_{\bblambda,k} \left( \E_{\bbh} \bbf(\bbphi(\bbh,\bbtheta_{k+1}),\bbh) -\bbx_{k+1} \right) \right]_+,\label{eq_pd_update3} \\
\bbmu_{k+1} &= \left[\bbmu_k - \gamma_{\bbmu,k} \bbg(\bbx_{k+1}) \right]_+, \label{eq_pd_update4}
\end{align}
with associated step sizes $\gamma_{\bblambda,k}, \bbgamma_{\bbmu,k} >0$. The gradient primal-dual updates in \eqref{eq_pd_update1}-\eqref{eq_pd_update4} successively move the primal and dual variables towards maximum and minimum points of the Lagrangian function, respectively.

The above gradient-based updates provide a natural manner by which to search for the optimal point of the dual function~$D_{\bbphi}$. However, direct evaluation of these updates requires both the knowledge of the functions $g_0, g, \bbf$, as well as the wireless channel distribution $m(\bbh)$. We cannot always assume this knowledge is available in practice. Indeed, existing models for, e.g., capacity functions, do not always capture the true physical performance in practice. The primal-dual learning method presented is thus considered here only as a baseline method upon which we can develop a completely model-free algorithm. The details of model-free learning are discussed further in the following section.

%\red{
%The primal-dual updates in \eqref{eq_pd_update1}-\eqref{eq_pd_update4} cannot be implemented directly in most practical systems, due to the requirement of model knowledge. In many settings, the models for, e.g. $g_0$ and $\bbf$ do not have known closed form expressions. Recall, for example, the optimal power allocation for capacity maximization in Example \ref{example_interference}. While the capacity of the wireless channel $\bbf$ can be modeled with the expression given in the first inequality constraint of \eqref{eq_interference_problem1}, the capacity in a real system will likely differ in a unknown manner. Likewise, without knowledge of the distribution $m(\bbh)$ we cannot evaluate $\E_{\bbh} \bbf$ directly. This makes the the direct computation of gradients and Jacobian used in \eqref{eq_pd_update1}-\eqref{eq_pd_update4} infeasible. In Section \ref{sec_model_free}, we discuss approaches towards estimating such gradients using online sampling of their respective functions. We first proceed, however, to discuss the details of a particular universal parameterization we employ known as deep neural networks.} \blue{Rewrite.}

%% file: analysis.tex
%!TEX root = root.tex

The duality gap is the difference $D_{\bbphi}^*-P_{\bbphi}^*$ between the dual and primal optima. For convex optimization problems this gap is null, which implies that one can work with the Lagrangian as in \eqref{eq_param_dual_function} without loss of optimality. The optimization problem in \eqref{eq_param_problem}, however, is not convex as it incorporates the nonconvex constraint in \eqref{eqn_ergo_avg_learning}. We will show here that despite the presence of this nonconvex constraint the duality gap $D_{\bbphi}^*-P_{\bbphi}^*$ is small when using parametrizations that are near universal in the sense of Definition \ref{def_universal}. In proving this result we need to introduce some restrictions to the problem formulation that we state as assumptions next.

%%%%%%%%%%%%%%%%%%%%%%%%%%%%%%%%%%%%%%%%%%%%%%%%%%%%%%%%%%%%%%%%%%%%%%%%%%%%%%%%%%%%%%%%
%%%   A   S   S   U   M   P   T   I   O   N   %%%%%%%%%%%%%%%%%%%%%%%%%%%%%%%%%%%%%%%%%%
%%%%%%%%%%%%%%%%%%%%%%%%%%%%%%%%%%%%%%%%%%%%%%%%%%%%%%%%%%%%%%%%%%%%%%%%%%%%%%%%%%%%%%%%
% 
\begin{assumption}\label{assumption_nonatomic}
The probability distribution $m(\bbh)$ is nonatomic in $\ccalH$. I.e., for any set $\ccalE\subseteq\ccalH$ of nonzero probability there exists a nonzero probability strict subset $\ccalE'\subset\ccalE$ of lower probability, $ 0 < \mbE_\bbh(\ind{\ccalE'}) < \mbE_\bbh(\ind{\ccalE})$. \end{assumption}

%%%%%%%%%%%%%%%%%%%%%%%%%%%%%%%%%%%%%%%%%%%%%%%%%%%%%%%%%%%%%%%%%%%%%%%%%%%%%%%%%%%%%%%%
%%%   A   S   S   U   M   P   T   I   O   N   %%%%%%%%%%%%%%%%%%%%%%%%%%%%%%%%%%%%%%%%%%
%%%%%%%%%%%%%%%%%%%%%%%%%%%%%%%%%%%%%%%%%%%%%%%%%%%%%%%%%%%%%%%%%%%%%%%%%%%%%%%%%%%%%%%%
%
\begin{assumption}\label{assumption_slater}
Slater's condition hold for the unparameterized problem in \eqref{eq_problem} and for the parametrized problem in \eqref{eq_param_problem}. In particular, there exists variables $\bbx_0$ and $\bbp_0(\bbh)$ and a strictly positive scalar constant $s>0$ such that 
\begin{align}\label{eqn_assumption_slater}
  \Efphzero - \bbx_0 \geq  s\bbone.
\end{align} \end{assumption}

%%%%%%%%%%%%%%%%%%%%%%%%%%%%%%%%%%%%%%%%%%%%%%%%%%%%%%%%%%%%%%%%%%%%%%%%%%%%%%%%%%%%%%%%
%%%   A   S   S   U   M   P   T   I   O   N   %%%%%%%%%%%%%%%%%%%%%%%%%%%%%%%%%%%%%%%%%%
%%%%%%%%%%%%%%%%%%%%%%%%%%%%%%%%%%%%%%%%%%%%%%%%%%%%%%%%%%%%%%%%%%%%%%%%%%%%%%%%%%%%%%%%
%
\begin{assumption}\label{assumption_nondecreasing_utilities}
The objective utility function $g_0(\bbx)$ is monotonically non-decreasing in each component. I.e., for any $\bbx\leq\bbx'$ it holds $g_0(\bbx)\leq g_0(\bbx')$.
\end{assumption}

%%%%%%%%%%%%%%%%%%%%%%%%%%%%%%%%%%%%%%%%%%%%%%%%%%%%%%%%%%%%%%%%%%%%%%%%%%%%%%%%%%%%%%%%
%%%   A   S   S   U   M   P   T   I   O   N   %%%%%%%%%%%%%%%%%%%%%%%%%%%%%%%%%%%%%%%%%%
%%%%%%%%%%%%%%%%%%%%%%%%%%%%%%%%%%%%%%%%%%%%%%%%%%%%%%%%%%%%%%%%%%%%%%%%%%%%%%%%%%%%%%%%
%
\begin{assumption}\label{assumption_lipschitz}
The expected performance function $\mbE \left[ \fph \right]$ is expectation-wise Lipschitz on $\bbp(\bbh)$ for all fading realizations $\bbh \in \ccalH$. Specifically, for any pair of resource allocations $\bbp_1(\bbh)\in \ccalP$ and $\bbp_2(\bbh) \in \ccalP$ there is a constant $L$ such that
\begin{align}
  \mbE \| \bbf(\bbp_1(\bbh), \bbh) - \bbf(\bbp_2(\bbh),\bbh) \|_{\infty}
       \leq L \mbE \| \bbp_1(\bbh) - \bbp_2(\bbh)\|_{\infty} .
\end{align} \end{assumption}

%%%%%%%%%%%%%%%%%%%%%%%%%%%%%%%%%%%%%%%%%%%%%%%%%%%%%%%%%%%%%%%%%%%%%%%%%%%%%%%%%%%%%%%%
%%%   M   A   I   N       M   A   T   T   E   R   %%%%%%%%%%%%%%%%%%%%%%%%%%%%%%%%%%%%%%
%%%%%%%%%%%%%%%%%%%%%%%%%%%%%%%%%%%%%%%%%%%%%%%%%%%%%%%%%%%%%%%%%%%%%%%%%%%%%%%%%%%%%%%%
%
Although Assumptions \ref{assumption_nonatomic}-\ref{assumption_lipschitz} restrict the scope of problems \eqref{eq_problem} and \eqref{eq_param_problem}, they still allow consideration of most problems of practical importance. Assumption \ref{assumption_slater} simply states that service demands can be provisioned with some slack. We point that an inequality analogous to \eqref{eqn_assumption_slater} holds for the other constraints in \eqref{eq_problem} and \eqref{eq_param_problem}. However, it is only the slack $s$ that appears in the bounds we will derive. Assumption \ref{assumption_nondecreasing_utilities} is a restriction on the utilities $g_0(\bbx)$, namely that increasing performance values result in increasing utility. Assumption \ref{assumption_lipschitz} is a continuity statement on each of the dimensions of the expectation of the constraint function $\bbf$---we point out this is weaker than general Lipschitz continuity. Referring back to the problems discussed in Examples \ref{example_point_to_point}-\ref{example_fdma}, it is evident that they satisfy the monotonicity assumption in Assumption \ref{assumption_nondecreasing_utilities}. Furthermore, the continuity assumption in Assumption \ref{assumption_lipschitz} is immediatley satisfied by the continuous capacity function in Examples \ref{example_point_to_point} and \ref{example_interference}, and is also satisfied by the binary problem in Example \ref{example_fdma} due to the bounded expectation of the capacity function. 

Assumption \ref{assumption_nonatomic} states that there are no points of strictly positive probability in the distributions $m(\bbh)$. This requires that the fading state $\bbh$ take values in a dense set with a proper probability density -- no distributions with delta functions are allowed. This is the most restrictive assumption in principle if we consider systems with a finite number of fading states. We observe that in reality fading does take on a continuum of values, though the channel estimation algorithms may quantize estimates to a finite number of fading states. We stress, however, that the learning algorithm we develop in the proceeding sections does not depend upon this property, and may be directly applied to channels with discrete states.

The duality gap of the original (unparameterized) problem in \eqref{eq_problem} is known to be null -- see Appendix \ref{app_param_duality} and \cite{ribeiro2012optimal}. Given the validity of Assumptions \ref{assumption_nonatomic} - \ref{assumption_lipschitz} and using a parametrization that is nearly universal in the sense of Definition  \ref{def_universal}, we show that the duality/parametrization gap $|D_{\bbphi}^*-P^*|$ between problems~\eqref{eq_problem} and~\eqref{eq_param_dual0} is small as we formally state next.

%%%%%%%%%%%%%%%%%%%%%%%%%%%%%%%%%%%%%%%%%%%%%%%%%%%%%%%%%%%%%%%%%%%%%%%%%%%%%%%%%%%%%%%%
%%%   T   H   E   O   R   E   M   %%%%%%%%%%%%%%%%%%%%%%%%%%%%%%%%%%%%%%%%%%%%%%%%%%%%%%
%%%%%%%%%%%%%%%%%%%%%%%%%%%%%%%%%%%%%%%%%%%%%%%%%%%%%%%%%%%%%%%%%%%%%%%%%%%%%%%%%%%%%%%%
%
\begin{theorem}\label{theorem_param_duality}
Consider the parameterized resource allocation problem in \eqref{eq_param_problem} and its Lagrangian dual in \eqref{eq_param_dual0} in which the parametrization $\bbphi$ is $\eps$-universal in the sense of Definition \ref{def_universal}. If Assumptions \ref{assumption_nonatomic}--\ref{assumption_lipschitz} hold, then the dual value~$D^*_{\bbphi}$ is bounded by
\begin{equation}\label{eq_theorem_param_duality}
   P^* - \|\bblambda^*\|_1 L \eps 
     \  \leq\  D^*_{\bbphi} 
     \  \leq\  P^*,
\end{equation}
where the multiplier norm $\|\bblambda^*\|_1$ can be bounded as 
\begin{equation}\label{eq_lambda_norm_bound} 
	\norm{\bblambda^*}_1 \leq \frac{P^* - g_0(\bbx_0)}{s}
		< \infty, 
\end{equation}
in which $\bbx_0$ is the strictly feasible point of Assumption \ref{assumption_slater}. \end{theorem}

%%%%%%%%%%%%%%%%%%%%%%%%%%%%%%%%%%%%%%%%%%%%%%%%%%%%%%%%%%%%%%%%%%%%%%%%%%%%%%%%%%%%%%%%
%%%   P   R   O   O   F   %%%%%%%%%%%%%%%%%%%%%%%%%%%%%%%%%%%%%%%%%%%%%%%%%%%%%%%%%%%%%%
%%%%%%%%%%%%%%%%%%%%%%%%%%%%%%%%%%%%%%%%%%%%%%%%%%%%%%%%%%%%%%%%%%%%%%%%%%%%%%%%%%%%%%%%
%
\begin{myproof} See Appendix \ref{app_param_duality}. \end{myproof}

%%%%%%%%%%%%%%%%%%%%%%%%%%%%%%%%%%%%%%%%%%%%%%%%%%%%%%%%%%%%%%%%%%%%%%%%%%%%%%%%%%%%%%%%
%%%   M   A   I   N       M   A   T   T   E   R   %%%%%%%%%%%%%%%%%%%%%%%%%%%%%%%%%%%%%%
%%%%%%%%%%%%%%%%%%%%%%%%%%%%%%%%%%%%%%%%%%%%%%%%%%%%%%%%%%%%%%%%%%%%%%%%%%%%%%%%%%%%%%%%
%
Given any near-universal parameterization that achieves $\epsilon$-accuracy with respect to all resource allocation policies in $\ccalP$, Theorem \ref{theorem_param_duality} establishes an upper and lower bound on the dual value in \eqref{eq_param_dual0} relative to the optimal primal of the original problem in \eqref{eq_problem}. The dual value is not greater than $P^*$ and, more importantly, not worse than a bias on the order of $\eps$. These bounds justify the use of the parametrized dual function in \eqref{eq_param_dual_function} as a means of solving the (unparameterized) wireless resource allocation problem in \eqref{eq_problem}. Theorem \ref{theorem_param_duality} shows that there exist a set of multipliers -- those that attain the optimal dual value $D^*_{\bbphi}$ -- that yield a problem that is within $\ccalO(\eps)$ of optimal.

It is interesting to observe that the duality gap $P^* -  D^*_{\bbphi} \leq \|\bblambda^*\|_1 L \eps$ has a very simple dependance on problem constants. The $\eps$ factor comes from the error of approximating arbitrary resource allocations $\bbp(\bbh)$ with parametrized resource allocations $\bbphi(\bbh,\bbtheta)$. The Lipschitz constant $L$ translates this difference into a corresponding difference between the functions $\fph$ and $\fphih$. The norm of the Lagrange multiplier $\|\bblambda^*\|_1$ captures the sensibility of the optimization problem with respect to perturbations, which in this case comes from the difference between $\fph$ and $\fphih$. This latter statement is clear from the bound in \eqref{eq_lambda_norm_bound}. For problems in which the constraints are easy to satisfy, we can find feasible points close the optimum so that $P^* - g_0(\bbx_0) \approx 0$ and $\bbs$ is not too small. For problems where constraints are difficult to satisfy, a small slack $\bbs$ results in a meaningful variation in $P^* - g_0(\bbx_0)$ and a large value for the ratio $[P^* - g_0(\bbx_0)] /s$. We point out that~\eqref{eq_lambda_norm_bound} is a classical bound in optimization theory that we include here for completeness.

%% file: model_free.tex
%!TEX root = root.tex

%\blue{The functions $g_0$ and $\bbg$ are under our control. We know them well. It's $\bbf$ that has to be estimated only. This will simplify your exposition and equations. There are only two equations you need to write. The other two you can write as before.}\green{Mark: This simplifies the presentation but not dramatically, so I don't necessarily agree that its not worth generalizing to a complete model-free method. I can think of cases, for instance in control, in which there are cost functions---e.g. drift---that are not known.}

In this section, we consider that often in practice, we do not have access to explicit knowledge of the functions $g_0$, $\bbg$, and~$\bbf$, along with the distribution $m(\bbh)$, but rather observe noisy estimates of their values at given operating points. While this renders the direct implementation of the standard primal-dual updates in \eqref{eq_pd_update1}-\eqref{eq_pd_update4} impossible, given their reliance on gradients that cannot be evaluated, we can use these updates to develop a \emph{model-free} approximation. Consider that given any set of iterates and channel realization $\{ \tbtheta, \tbx, \tbh\}$, we can observe stochastic function values $\hat{g}_0(\tbx)$, $\hbg(\tbx)$, and $\hbf(\tbh,\bbphi(\tbh,\tbtheta))$. For example, we may pass test signals through the channel at a given power or bandwidth to measure its capacity or packet error rate. These observations are, generally, unbiased estimates of the true function values. 

We can then replace the updates in \eqref{eq_pd_update1}-\eqref{eq_pd_update4} with so-called zeroth-ordered updates, in which we construct estimates of the function gradients using observed function values. Zeroth-ordered gradient estimation can be done naturally with the method of finite differences, in which unbiased gradient estimators at a given point are constructed through random perturbations. Consider that we draw random perturbations $\hbx_1, \hbx_2 \in \reals^u$  and $\hbtheta \in \reals^q$ from a standard Gaussian distribution and a random channel state $\hbh$ from $m(\bbh)$.  Finite-difference gradients estimates $\widehat{\nabla} g_0$, $\widehat{\nabla} \bbg$, and $\widehat{\nabla_{\bbtheta}} \E \bbf$ can be constructed using function observations at given points $\{\bbx_0, \bbtheta_0\}$ and the sampled perturbations as
\begin{align}
\widehat{\nabla} g_0&(\bbx_0) := \frac{ \hat{g}_0(\bbx_0 + \alpha_{1} \hbx_1) - \hat{g}_0(\bbx_0)}{\alpha_{1}} \hbx_1, \label{eq_fd_1} \\
\widehat{\nabla} \bbg&(\bbx_0) := \frac{ \hbg(\bbx_0 + \alpha_{3} \hbx_2) - \hbg(\bbx_0)}{\alpha_{3}} \hbx_2^T, \label{eq_fd_2} \\
\widehat{ \nabla_{\bbtheta}} \E[\bbf(\bbphi(\bbh,&\bbtheta_0),\bbh)] := \label{eq_fd_3} \\
 &\frac{ \hbf(\bbphi(\hbh,\bbtheta_0 + \alpha_{2} \hbtheta),\hbh) - \hbf(\bbphi(\hbh,\bbtheta_0),\hbh)}{\alpha_{2}} \hbtheta^T,  \nonumber 
\end{align}
where we define scalar step sizes $\alpha_{1}, \alpha_{2},\alpha_{3} >0$. The expressions in \eqref{eq_fd_1}-\eqref{eq_fd_3} provide estimates of the gradients that can be computed using only two function evaluations. Indeed, the finite difference estimators can be shown to be unbiased, meaning that that they coincide with the true gradients in expectation---see, e.g., \cite{nesterov2011random}. Note also in \eqref{eq_fd_3} that, by sampling both the function $\bbf$ and a channel state $\hbh$, we directly estimate the expectation $\E_{\bbh} \bbf$. We point out that these estimates can be further improved by using batches of $B$ samples, $\{  \hbx_1^{(b)}, \hbx_2^{(b)}, \hbtheta^{(b)}, \hbh^{(b)} \}_{b=1}^B$, and averaging over the batch. We focus on the simple stochastic estimates in \eqref{eq_fd_1}-\eqref{eq_fd_3}, however, for clarity of presentation.

Note that, while using the finite difference method to estimate the gradients of the deterministic function $g_0(\bbx)$ and $\bbg(\bbx)$ is relatively simple, estimating the stochastic policy function $\E_{\bbh} \bbf(\bbphi(\bbh,\bbtheta),\bbh)$ is often a computational burden in practice when the parameter dimension $q$ is very large---indeed, this is often the case in, e.g., deep neural network models. An additional complication arises in that the function must be observed multiple times for the same sample channel state $\hbh$ to obtain the perturbed value. This might be impossible to do in practice if the channel state changes rapidly. There indeed exists, however, an alternative model free approach for estimating the gradient of a policy function, which we discuss in the next subsection.

%%%%%%%%%%%%%%%%%%%%%%%%%%%%%%%%%%%%%%%%%%%%%%%%%%%%%%%%%%%%%%%%%
%%%%%%%%%%%%%%%%%%%%%%%%%
%%%%%%% S U B S E C T I O N %%%%%%%%%%%
%%%%%%%%%%%%%%%%%%%%%%%%%%
%%%%%%%%%%%%%%%%%%%%%%%%%%%%%%%%%%%%%%%%%%%%%%%%%%%%%%%%%%%%%%%%%%%%%
\subsection{Policy gradient estimation}\label{sec_policy_grad}

The ubiquity of computing the gradients of policy functions such as $\nabla_{\bbtheta} \E\bbf(\bbphi(\bbh,\bbtheta),\bbh)$ in machine learning problems has motivated the development of a more practical estimation method. The so-called \emph{policy gradient} method exploits a likelihood ratio property found in such functions to allow for an alternative zeroth ordered gradient estimate. To derive the details of the policy gradient method, consider that a deterministic policy $\bbphi(\bbh,\bbtheta)$ can be reinterpreted as a \emph{stochastic} policy drawn from a distribution with density function $\pi(\bbp)$ defined with a delta function, i.e., $\pi_{\bbh,\bbtheta}(\bbp) = \delta(\bbp - \bbphi(\bbh,\bbtheta))$. It can be shown that the Jacobian of the policy constraint function $ \E_{\bbh,\bbphi}[\bbf(\bbphi(\bbh,\bbtheta),\bbh)]$ with respect to $\bbtheta$ can be rewritten using this density function as 
\begin{equation}\label{eq_policy_gradient}
\nabla_{\bbtheta} \mathbb{E}_{\bbh} \bbf( \bbphi(\bbh,\bbtheta),\bbh) =  \mathbb{E}_{\bbh, \bbp} [\bbf( \bbp,\bbh) \nabla_{\bbtheta} \log \pi_{\bbh,\bbtheta}(\bbp)^T],
\end{equation}
where $\bbp$ is a random variable drawn from distribution $\pi_{\bbh,\bbtheta}(\bbp)$---see, e.g., \cite{sutton2000policy}. Observe in \eqref{eq_policy_gradient} that the computation of the Jacobian reduces to a function evaluation multiplied by the gradient of the policy distribution $\nabla_{\bbtheta} \log \pi_{\bbh,\bbtheta}(\bbp)$. Indeed, in the deterministic case where the distribution is a delta function, the gradient cannot be evaluated without knowledge of $m(\bbh)$ and $\bbf$. However, we may approximate the delta function with a known density function centered around $\bbphi(\bbh,\bbtheta)$, e.g., Gaussian distribution. If an analytic form for $\pi_{\bbh,\bbtheta}(\bbp)$ is known, we can estimate $\nabla_{\bbtheta} \mathbb{E}_{\bbh} \bbf(\bbphi(\bbh,\bbtheta),\bbh)$ by instead directly estimating the left-hand side of \eqref{eq_policy_gradient}. In the context of reinforcement learning, this is called the REINFORCE method \cite{sutton2000policy}. By using the previous function observations, we can obtain the following policy gradient estimate,
%when the policy distribution $\bbpi_{\bbh,\bbtheta}$ follows a known, analytic form. This approach is detailed in \cite{sutton2000policy} and is known as REINFORCE style policy gradients due to its popular usage in reinforcement learning. The key idea is that by using stochastic policies, we may be able to compute the expected Jacobian using only function value evaluations and the Jacobian of the distribution $\bbpi_{\bbh,\bbtheta}$. The Jacobian of the policy constraint function $ \E_{\bbh,\bbtheta}\bbf(\bbh,\bbp_{\bbtheta})$ with respect to $\bbtheta$ can then be written as
%
%\begin{equation}\label{eq_policy_gradient}
%\nabla_{\bbtheta} \mathbb{E}_{\bbh, \bbtheta} \bbf(\bbh, \bbp_{\bbtheta}) =  \mathbb{E}_{\bbh, \bbp} [\bbf(\bbh, \bbp_{\bbtheta}) \nabla_{\bbtheta} \log \bbpi_{\bbh,\bbtheta}(\bbp_{\bbtheta})].
%\end{equation}
%
%Observe in \eqref{eq_policy_gradient} that the computation of the Jacobian reduces to a function evaluation multiplied by the Jacobian $\nabla_{\bbtheta} \log \bbpi_{\bbphi(\bbh,\bbtheta)}(\bbp_{\bbtheta} )$. Indeed, such an equivalence is only useful in the model-free setting when the the density function $\bbpi_{\bbh,\bbtheta}$---and likewise its Jacobian---are themselves known. In this case, however, we can now estimate the Jacobian using the function observations as
%
\begin{equation}\label{eq_policy_gradient_s}
\widehat{\nabla_{\bbtheta}} \mathbb{E}_{\bbh} \bbf(\bbphi(\bbh,\bbtheta),\bbh) =  \hbf( \hbp_{\bbtheta},\hbh) \nabla_{\bbtheta} \log \pi_{\hbh,\bbtheta}(\hbp_{\bbtheta} )^T,
\end{equation}
where $\hbp_{\bbtheta}$ is a sample drawn from the distribution $\pi_{\bbh,\bbtheta}(\bbp)$.

The policy gradient estimator in \eqref{eq_policy_gradient_s} can be taken as an alternative to the finite difference approach in \eqref{eq_fd_3} for estimating the gradient of the policy constraint function, provided the gradient of the density function $\bbpi$ can itself be evaluated. Observe in the above expression that the policy gradient approach replaces a sampling of the parameter $\bbtheta \in \reals^q$ with a sampling of a resource allocation $\bbp \in \reals^m$. This is indeed preferable for many sophisticated learning models in which $q \gg m$. We stress that while policy gradient methods are preferable in terms of sampling complexity, they come at the cost of placing an additional approximation through the use of a stochastic policy analytical density functions $\bbpi$.

\subsection{Model-free primal-dual method}

 Using the gradient estimates in \eqref{eq_fd_1}-\eqref{eq_fd_3}---or \eqref{eq_policy_gradient_s}---we can derive a model-free, or zeroth-ordered, stochastic updates to replace those in \eqref{eq_pd_update1}-\eqref{eq_pd_update4}. By replacing all function evaluations with the function observations and all gradient evaluations with the finite difference estimates, we can perform the following stochastic updates
\begin{align}
\bbtheta_{k+1} &= \mathsf{P}_{\Theta}\left[\bbtheta_k + \gamma_{\bbtheta,k} \widehat{\nabla_{\bbtheta}}\mathbb{E}_{\bbh} \bbf(\bbphi(\bbh,\bbtheta_{k}),\bbh) \bblambda_k  \right], \label{eq_spd_update1} \\
\bbx_{k+1} &= \mathsf{P}_{\ccalX}\left[\bbx_k  +  \gamma_{\bbx,k} (\widehat{\nabla g_0}(\bbx) + \widehat{\nabla \bbg}(\bbx_{k})\bbmu_k - \bblambda_k) \right] , \label{eq_spd_update2} \\
\bblambda_{k+1} &= \left[ \bblambda_k - \gamma_{\bblambda,k} \left( \hbf(\bbphi(\hbh_k,\bbtheta_{k+1}),\hbh_k) -\bbx_{k+1} \right) \right]_+\label{eq_spd_update3} \\
\bbmu_{k+1} &= \left[\bbmu_k - \gamma_{\bbmu,k} \hbg(\bbx_{k+1}) \right]_+. \label{eq_spd_update4}
\end{align}

The expressions in \eqref{eq_spd_update1}-\eqref{eq_spd_update4} provides means of updating both the primal and dual variables in a primal-dual manner without requiring any explicit knowledge of the functions or channel distribution through observing function realizations at the current iterates. We may say this method is \emph{model-free} because all gradients used in the updates are constructed entirely from \emph{measurements}, rather than analytic computation done via model knowledge. The complete model-free primal-dual learning method can be summarized in Algorithm \ref{alg:learning}. The method is initialized in Step 1 through the selection of parameterization model $\bbphi(\bbh,\bbtheta)$ and form of the stochastic policy distribution $\pi_{\bbh,\bbtheta}$ and in Step 2 through the initialization of the primal and dual variables. For every step $k$, the algorithm begins in Step 4 by drawing random samples (or batches) of the primal and dual variables. In Step 5, the model functions are sampled at both the current primal and dual iterates and at the sampled points. These function observations are then used in Step 6 to form gradient estimates via finite difference (or policy gradient). Finally, in Step 7 the model-free gradient estimates are used to update both the primal and dual iterates.

%%%%%%%%%%%%%%%%%%%%%%%%%%%%%%%%%%%%%%%%%%%%%%%%%%%%%%%%%%%%%%%%
%%%%   A   L   G   O   R   I   T   H   M   %%%%%%%%%%%%%%%%%%%%%
%%%%%%%%%%%%%%%%%%%%%%%%%%%%%%%%%%%%%%%%%%%%%%%%%%%%%%%%%%%%%%%%
{\linespread{1.3}
\begin{algorithm}[t] \begin{algorithmic}[1]
\STATE \textbf{Parameters:} Policy model $\bbphi(\bbh,\bbtheta)$ and distribution form $\pi_{\bbh,\bbtheta}$ 
\STATE \textbf{Input:} Initial states $\bbtheta_0, \bbx_0, \bblambda_0 \bbmu_0$
\FOR [main loop]{$k = 0,1,2,\hdots$}
      \STATE Draw samples $\{  \hbx_1, \hbx_2, \hbtheta, \hbh_k \}$, or in batches of size $B$
      \STATE Obtain random observation of function values $\hat{g}_0, \hbf$ $\hbg$ at current and sampled iterates
      \STATE Compute gradient estimates $\widehat{\nabla} g_0(\bbx)$,  $\widehat{\nabla} \bbg(\bbx)$, $\widehat{ \nabla_{\bbtheta}} \E_{\bbh,\bbphi}\bbf(\bbphi(\bbh,\bbtheta),\bbh)$,  [cf. \eqref{eq_fd_1}-\eqref{eq_fd_3} or \eqref{eq_policy_gradient_s}]
      \STATE Update primal and dual variables [cf. \eqref{eq_spd_update1}-\eqref{eq_spd_update4}] \ 
         \begin{align}
	\bbtheta_{k+1} &= \mathsf{P}_{\Theta}\left[\bbtheta_k + \gamma_{\bbtheta,k} \widehat{\nabla_{\bbtheta}} \mathbb{E}_{\bbh}\bbf(\bbphi(\hbh_k,\bbtheta_{k}),\hbh_k) 	\bblambda_k  \right], \nonumber \\
	\bbx_{k+1} &= \mathsf{P}_{\ccalX}\left[\bbx_k  +  \gamma_{\bbx,k} (\widehat{\nabla g_0}(\bbx) + \widehat{\nabla \bbg}(\bbx_{k})\bbmu_k - \bblambda_k) 		\right] , \nonumber \\
	\bblambda_{k+1} &= \left[ \bblambda_k - \gamma_{\bblambda,k} \left( \hbf(\bbphi(\hbh_k,\bbtheta_{k+1}),\hbh_k) -\bbx_{k+1} \right) \right]_+	\nonumber \\
\bbmu_{k+1} &= \left[\bbmu_k - \gamma_{\bbmu,k} \hbg(\bbx_{k+1}) \right]_+. \nonumber
\end{align}
\ENDFOR

\end{algorithmic}
\caption{Model-Free Primal-Dual Learning}\label{alg:learning} \end{algorithm}}

We briefly comment on the known convergence properties of the model-free learning method in \eqref{eq_spd_update1}-\eqref{eq_spd_update4}. Due to the non-convexity of the Lagrangian defined in \eqref{eq_param_lagrangian}, the stochastic primal-dual descent method will converge only to a local optima and is not guaranteed to converge to a point that achieves $D_{\bbtheta}^*$. These are indeed the same convergence properties of general \emph{unconstrained} non-convex learning problems as well. We instead demonstrate through numerical simulations the performance of the proposed learning method in practical wireless resource allocation problems in the proceeding section. 

\begin{remark}\label{remark_convergence}\normalfont
The algorithm presented in Algorithm \ref{alg:learning} is generic in nature and can be supplemented with more sophisticated learning techniques that can improve the learning process. Some examples include the use of entropy regularization to improve policy optimization in non-convex problems \cite{haarnoja2017reinforcement}. Policy optimization can also be improved using actor-critic methods \cite{sutton1998reinforcement}, while the use of a model function estimate to obtain ``supervised'' training signals can be used to initialize the parameterization vector $\bbtheta$. The use of such techniques in optimal wireless design are not explored in detail here and left as the study of future work. 
\end{remark}

%% file: neural_networks.tex
We have so far discussed a theoretical and algorithm means of learning in wireless systems by employing any near universal parametrization as defined in Definition \ref{def_universal}. In this section, we restrict our attention to the increasingly popular set of parameterizations known as \emph{deep neural networks} (DNNs), which are often observed in practice to exhibit strong performance in function approximation. In particular, we discuss the details of the DNN parametrization model and both the theoretical and practical implications within our constrained learning framework.

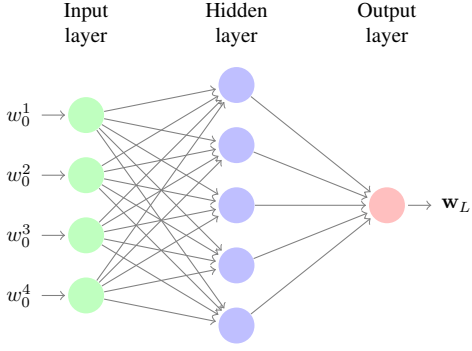
\begin{figure}
\centering
\input{nn1.tex}
\caption{Typical architecture of fully-connected deep neural network.}
\label{fig_nn}
\end{figure}

The exact form of a particular DNN is described by what is commonly referred to as its \emph{architecture}. The architecture consists of a prescribed number of layers, each of which consisting of a linear operation followed by a point-wise nonlinearity---also known as an activation function.  In particular, consider a DNN with $L$ layers, labelled $l=1,\hdots,L$ and each with a corresponding dimension $q_l$. The layer $l$ is  defined by the linear operation $\bbW_l \in \reals^{q_{l-1} \times q_l}$ followed by a non-linear activation function $\bm{\sigma}_{l}: \reals^{q_l} \rightarrow \reals^{q_l}$. If layer $l$ receives as an input from the $l-1$ layer $\bbw_{l-1} \in \reals^{q_{l-1}}$, the resulting output $\bbw_{l} \in \reals^{q_l}$ is then computed as $\bbw_l := \bm{\sigma}_{l}(\bbW_l \bbw_{l-1})$.  The final output of the DNN, $\bbw_L$, is then related to the input $\bbw_0$ by propagating through each later of the DNN as $\bbw_L =  \bm{\sigma}_{L}(\bbW_{L} (\bm{\sigma}_{L-1}(\bbW_{L-1}(\hdots(\bm{\sigma}_{1}(\bbW_{1}\bbw_0))))))$. 

An illustration of a fully-connected example DNN architecture is given in Figure \ref{fig_nn}. In this example, the inputs $\bbw$ are passed through a single hidden layer, following which is an output layer. The grey lines between layers reflect the linear transformation $\bbW_l$, while each node contains an additional element-wise activation function $\bm{\sigma}_l$. This general  DNN structure has been observed to have remarkable generalization and approximation properties in a variety of functional parameterization problems.

The goal in learning DNNs in general then reduces to learning the linear weight functions $\bbW_1, \hdots, \bbW_L$. Common choices of activation functions $\bm{\sigma}_{l}$ include a sigmoid function, a rectifier function (commonly referred to as ReLu), as well as a smooth approximation to the rectifier known as softplus. For the parameterized resource allocation problem in \eqref{eq_param_problem}, the policy $\bbphi(\bbh,\bbtheta)$ can be defined by an $L$-layer DNN as
\begin{equation}\label{eq_policy_dnn}
\bbphi(\bbh,\bbtheta) := \bm{\sigma}_{L}(\bbW_{L} (\bm{\sigma}_{L-1}(\bbW_{L-1}(\hdots(\bm{\sigma}_{1}(\bbW_{1}\bbh)))))),
\end{equation}
where $\bbtheta \in \reals^q$ contains the entries of $\{ \bbW_l \}_{l=1}^L$ with $q = \sum_{l=1}^{L-1} q_l q_{l+1} $. Note that $q_1 = n$ by construction. 

To contextualize the primal-dual algorithm in \eqref{eq_pd_update1}-\eqref{eq_pd_update4} with respect to traditional neural network training, observe that the update in \eqref{eq_pd_update1} requires computation of the gradient $\nabla_{\bbtheta} \E_{\bbh} \bbf(\bbh,\bbphi(\bbh,\bbtheta))$. Using the chain rule, this can be expanded as
\begin{align}\label{jacobian}
\nabla_{\bbtheta} \E_{\bbh} \bbf(&\bbphi(\bbh,\bbtheta_{k}),\bbh) =  \\ 
&\qquad \nabla_{\bbphi}\E_{\bbh} \bbf(\bbphi(\bbh,\bbtheta_{k}),\bbh) \nabla_{\bbtheta} \bbphi (\bbh, \bbtheta_k). \nonumber
\end{align}
Thus, the computation of the full gradient requires evaluating the gradient of the policy function $\bbf$ as well as the gradient of the DNN model $\bbphi$. For the DNN structure in \eqref{eq_policy_dnn}, the evaluation of $\nabla_{\bbtheta} \bbphi$ may itself also require a chain rule expansion to compute partial derivatives at each layer of the network. This process of performing gradient descent to find the optimal weights in the DNN is commonly referred to as backpropogation. 

We further take note how our learning approach differs from a more traditional, \emph{supervised} training of DNNs. As in \eqref{jacobian}, the backpropogation is performed with respect to the given policy constraint function $\bbf$, rather than with respect to a Euclidean loss function over a set of given training data. Furthermore, due to the constraints, the backpropogation step in \eqref{eq_pd_update1} is performed in sequence with the more standard primal and dual variable updates in \eqref{eq_pd_update2}-\eqref{eq_pd_update4}. In this way, the DNN is trained indirectly \emph{within} the broader optimization algorithm used to solve \eqref{eq_param_problem}. This is in contrast with other approaches of training DNNs in constrained wireless resource allocation problems---see, e.g. \cite{sun2017learning,lei2017deep,lee2018deep}---which train a DNN to approximate the complete constrained maximization function in \eqref{eq_problem} directly. Doing so requires the ability to solve \eqref{eq_problem} either exactly or approximately enough times to acquire a labeled training set. The primal-dual learning approach taken here is preferable in that it does not require the use of training data. The dual problem can be seen as a simplified reinforcement learning problem---one in which the actions do not affect the next state.

For DNNs to be valid parametrization with respect to the result in Theorem \ref{theorem_param_duality}, we must first verify that they satisfy the near-universality property in Definition \ref{def_universal}. Indeed, deep neural networks are popular parameterizations for arbitrary functions precisely due to the richness inherent in \eqref{eq_policy_dnn}, which in general grows richer with number of layers $L$ and associated layer sizes $q_l$. This richness property of DNNs has been the subject of mathematical study and formally referred to as a complete universal function approximation \cite{cybenko1989approximation, hornik1991approximation}. In words, this property implies that a large class of functions~$\bbp(\bbh)$ can be approximated with \emph{arbitrarily} small accuracy $\eps$ using a DNN parameterization of the form in \eqref{eq_policy_dnn} with only a single layer of arbitrarily large size. With this property in mind,
%
%we verify that DNNs can in fact satisfy Definition \ref{def_universal} with arbitrarily small $\eps$ in the following proposition.
%
%\begin{proposition}\label{prop_approx}
%Let~$\mathcal{P}$ be the set of continuous policies, i.e., $\mathcal{P} = \mathcal{C}(\mathcal{H})$. Then, for large enough number of hidden nodes~$q$, the DNN parametrization $\bbphi(\bbh,\bbtheta)$ in \eqref{eq_policy_dnn} is fully universal, i.e., it satisfies the near-universality property in Definition \ref{def_universal} for all~$\eps > 0$.
%\end{proposition}
%%
%\begin{myproof}
%See~\cite[Theorem 2]{hornik1991approximation}.% Appendix \ref{sec_dnn_proof}.
%\end{myproof}
%%
%The result in Proposition \ref{prop_approx} is pivotal because it states that a DNN can achieve arbitrary approximation accuracy $\eps$ to any continuous allocation given sufficiently large layers. Thus, the bounded suboptimality result in Theorem \ref{theorem_param_duality} holds for the DNN parametrization. Furthermore, the fact that it holds for arbitrarily small~$\eps$ motivates the extension of this result to the case of arbitrarily small suboptimality. With this property in mind,
we can present the following theorem that extends the result in Theorem \ref{theorem_param_duality} in the special case of DNNs.
\begin{theorem}\label{cor_param_duality}
Consider the DNN parametrization $\bbphi(\bbh,\bbtheta)$ in \eqref{eq_policy_dnn} with non-constant, continuous activation functions~$\bm{\sigma}_l$ for~$l = 1,\dots,L$. Define the vector of layer lengths $\bbq = [q_1; q_2; \hdots; q_L]$ and a DNN defined in \eqref{eq_policy_dnn} with lengths $\bbq$ as $\bbphi_{\bbq}(\bbh,\bbtheta)$. Now consider the set of possible $L$-layer DNN parameterization functions $\Phi := \{\bbphi_{\bbq}(\bbh,\bbtheta) \mid \bbq \in \mathbb{N}^L\}$.  If Assumptions \ref{assumption_nonatomic}--\ref{assumption_lipschitz} hold, then the optimal dual value of the parameterized problem satisfies
\begin{equation}\label{eq_cor_param_duality}
\inf_{\bbphi \in \Phi} \bbD^*_{\bbphi} = P^*.
\end{equation}
\end{theorem}
%%%%%%%%%%%%%%%%%%%%%%%%%%%%%%%%%%%%%%%%%%%%%%%%%%%%%%%%%%%%%%%%%
%%%%%%%%%%%%%%%%%%%%%%%%%
%%%%%%% P R O O F %%%%%%%%%%%
%%%%%%%%%%%%%%%%%%%%%%%%%%
%%%%%%%%%%%%%%%%%%%%%%%%%%%%%%%%%%%%%%%%%%%%%%%%%%%%%%%%%%%%%%%%%%%%%
\begin{myproof}
See Appendix \ref{sec_cor_proof}.
\end{myproof}
%%%%%%%%%%%%%%%%%%%%%%%%%%%%%%%%%%%%%%%%%%%%%%%%%%%%%%%%%%%%%%%%%

With Theorem \ref{cor_param_duality} we establish the null duality gap property of a resource allocation problem of the form in \eqref{eq_param_problem} given a DNN parameterization that achieves arbitrarily small function approximation accuracy as the dimension of the DNN parameter---i.e. the number of hidden nodes---grows to infinity. While such a parametrization is indeed guaranteed to exist through the universal function approximation theorem, one would require a DNN with arbitrarily large size to obtain such a network in practice. As such, the suboptimality bounds presented in Theorem \ref{theorem_param_duality}, which require only an DNN-approximation of given accuracy $\eps$ provide the more practical characterization of \eqref{eq_param_problem}, while the result in Theorem \ref{cor_param_duality} suggests DNNs can be used find parameterizations of arbitrarily strong accuracy.

%% file: nn1.tex
\def\layersep{2.5cm}

\begin{tikzpicture}[shorten >=1pt,->,draw=black!50, node distance=\layersep, scale=.8, transform shape]
    \tikzstyle{every pin edge}=[<-,shorten <=1pt]
    \tikzstyle{neuron}=[circle,fill=black!25,minimum size=17pt,inner sep=0pt]
    \tikzstyle{input neuron}=[neuron, fill=green!25];
    \tikzstyle{output neuron}=[neuron, fill=red!25];
    \tikzstyle{hidden neuron}=[neuron, fill=blue!25];
    \tikzstyle{annot} = [text width=4em, text centered]

    % Draw the input layer nodes
    \foreach \name / \y in {1,...,4}
    % This is the same as writing \foreach \name / \y in {1/1,2/2,3/3,4/4}
        \node[input neuron, pin=left: $w_0^{\y}$] (I-\name) at (0,-\y) {};

    % Draw the hidden layer nodes
    \foreach \name / \y in {1,...,5}
        \path[yshift=0.5cm]
            node[hidden neuron] (H-\name) at (\layersep,-\y cm) {};

    % Draw the output layer node
    \node[output neuron,pin={[pin edge={->}]right:$\bbw_L$}, right of=H-3] (O) {};

    % Connect every node in the input layer with every node in the
    % hidden layer.
    \foreach \source in {1,...,4}
        \foreach \dest in {1,...,5}
            \path (I-\source) edge (H-\dest);

    % Connect every node in the hidden layer with the output layer
    \foreach \source in {1,...,5}
        \path (H-\source) edge (O);

    % Annotate the layers
    \node[annot,above of=H-1, node distance=1cm] (hl) {Hidden layer};
    \node[annot,left of=hl] {Input layer};
    \node[annot,right of=hl] {Output layer};
\end{tikzpicture}

% End of code

%% file: numerical_results.tex
In this section, we provide simulation results on using the proposed primal-dual learning method to solve for DNN-parameterizations of resource allocation in a number of common problems in wireless communications that take the form in \eqref{eq_problem}. For the simulations performed, we employ a stochastic policy and implement the REINFORCE-style policy gradient described in Section \ref{sec_policy_grad}. In particular, we select the policy distribution $\pi_{\bbtheta,\bbh}$ as a truncated Gaussian distribution. The truncated Gaussian distribution has fixed support on the domain $[0,p_{\max}]$. The output layer of the DNN $\bbphi(\bbh,\bbtheta) \in \reals^{2m}$ is the set of $m$ means and standard deviations to specify the respective truncated Gaussian distributions, i.e. $\bbphi(\bbh,\bbtheta) := [\mu^1; \sigma^1; \mu^2; \sigma^2; \hdots; \mu^m; \sigma^m]$. Furthermore, to represent policies that are bounded on the support interval, the output of the last layer is fed into a scaled sigmoid function such that the mean lies in the area of support and the variance is no more than the square root of the support region. In the following experiments, this interval is [0, 10]. 

For updating the primal and dual variables, we use a batch size of 32. The primal dual method is performed with an exponentially decaying step size for dual updates and the ADAM optimizer \cite{kingma2014adam} for the DNN parameter update. Both updates start with a learning rate of 0.0005, while random channel conditions are generated with an exponential distribution with parameter $\lambda=2$ (to represent the square of a unit variance Rayleigh fading channel state). 

\subsection{Simple AWGN channel}

 %%%%%%%%%%%%%%%%%%%%%%%%%%%%%%%%
%%%%%%%%%% F I G U R E %%%%%%%%%%%%%%%%%
%%%%%%%%%%%%%%%%%%%%%%%%%%%%%%%%

\begin{figure}
\centering
\input{nn1_siso.tex}
\caption{Neural network architecture used for simple AWGN channel. Each channel state $h^i$ is fed into an independent SISO network with two hidden layers of size 8 and 4, respectively. The DNN outputs a mean $\mu^i$ and standard deviation $\sigma^i$ for a truncated Gaussian distribution.}
\label{fig_nn_siso}
\end{figure}
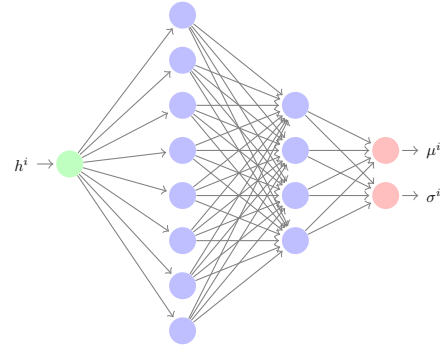

%\begin{figure}
%	\centering
%	\includegraphics[width=0.8\linewidth]{wireless_1.png} 
%	\caption{Neural network architecture used for simple AWGN channel. Each channel state $h_i$ is fed into an independent network for two hidden layers of size 8 and 4, respectively. The DNN outputs a mean $\mu_i$ and standard deviation $\sigma_i$ for a truncated Gaussian distribution.}
%	\label{fig_nn_siso}
%\end{figure}
%
%

 %%%%%%%%%%%%%%%%%%%%%%%%%%%%%%%%
%%%%%%%%%% F I G U R E %%%%%%%%%%%%%%%%%
%%%%%%%%%%%%%%%%%%%%%%%%%%%%%%%%

\begin{figure*}[t]
\centering
\includegraphics[width=0.32\linewidth, height=.17\textheight]{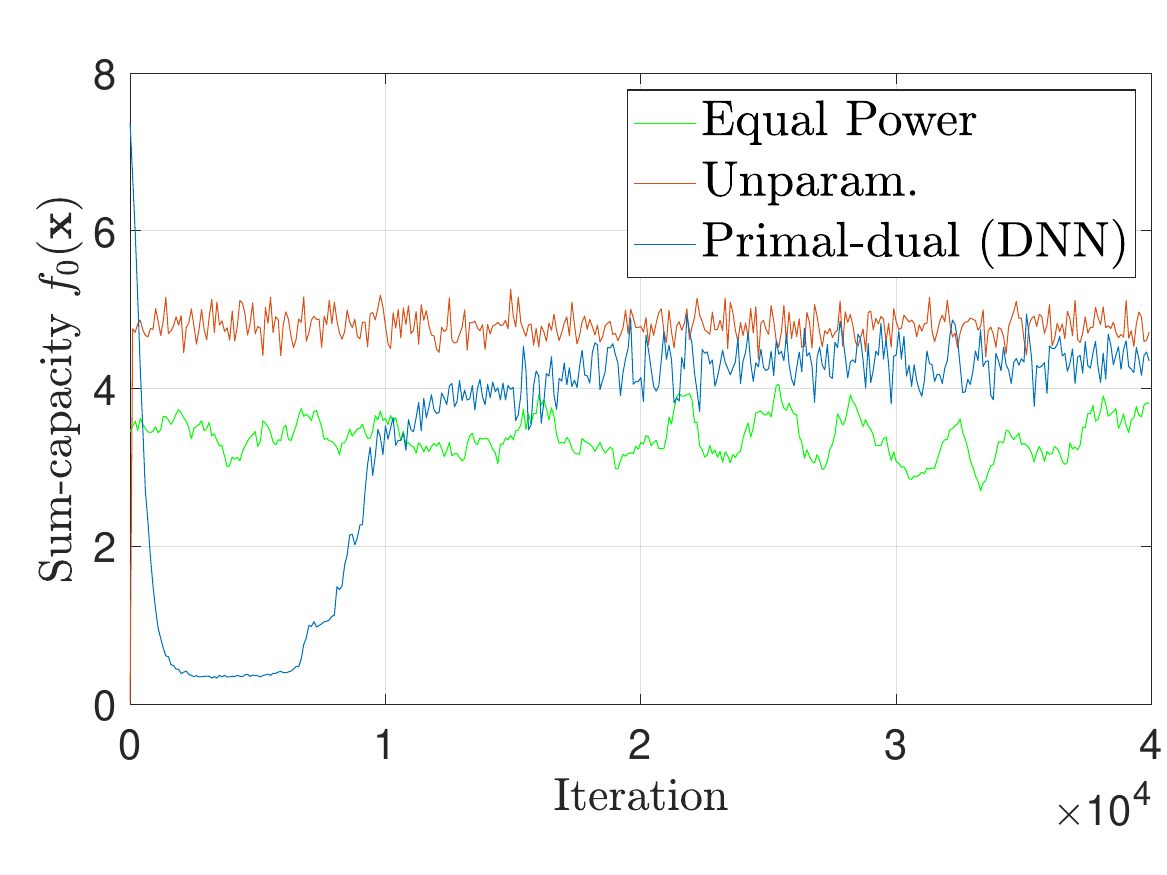} 
\includegraphics[width=0.32\linewidth, height=.17\textheight]{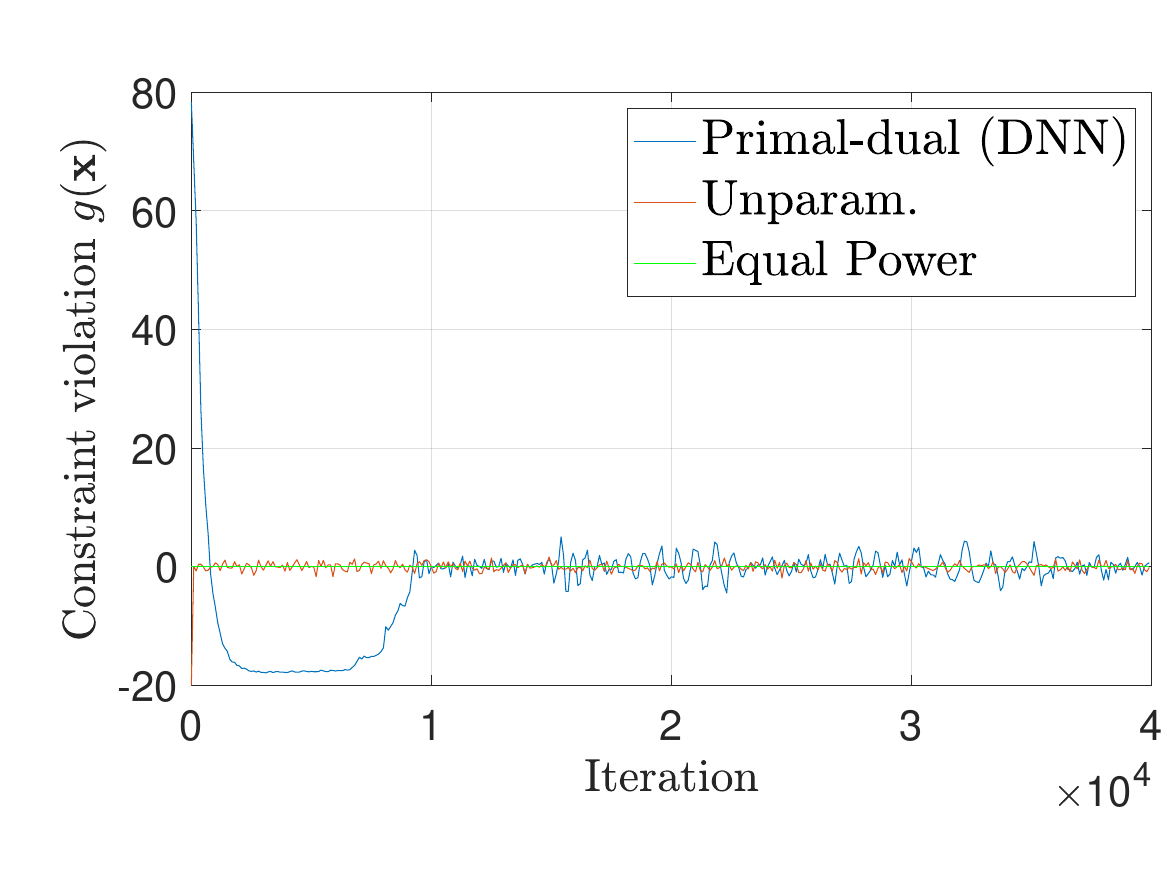} 
\includegraphics[width=0.32\linewidth, height=.17\textheight]{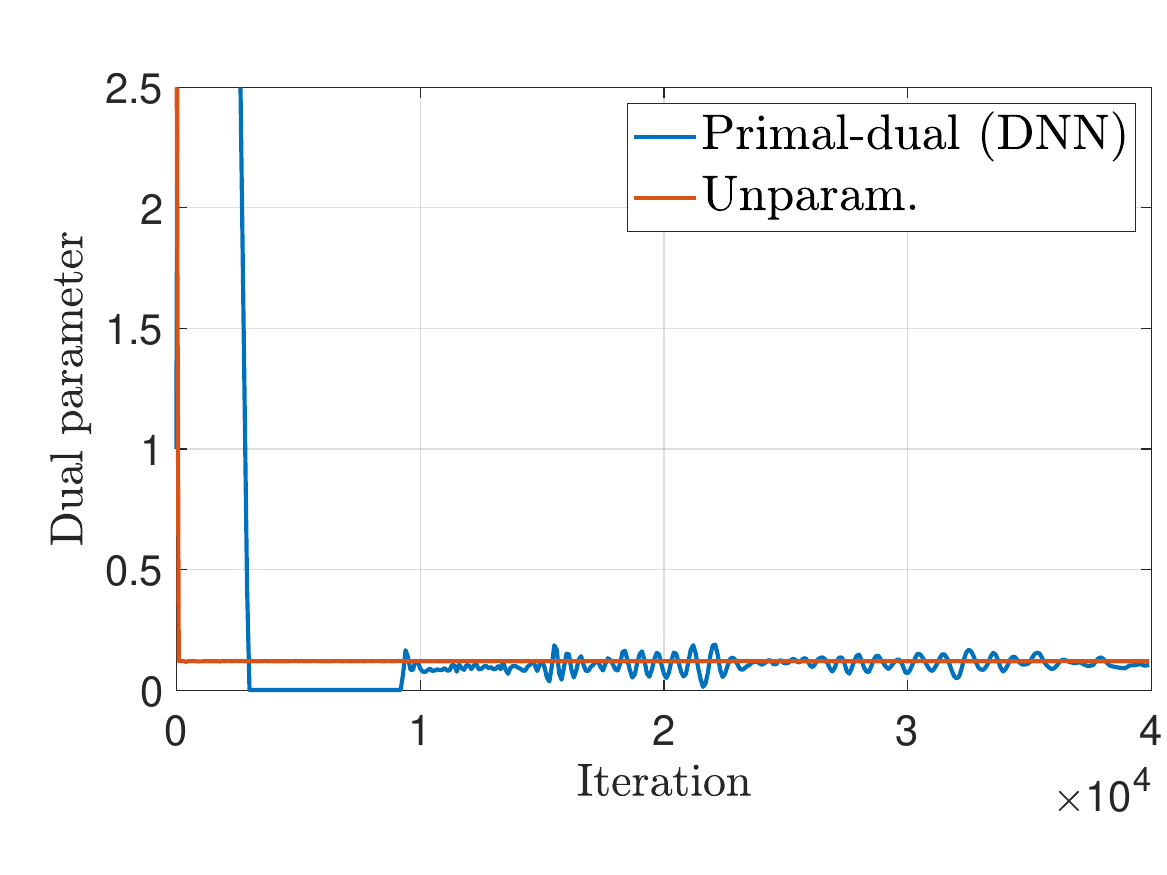} 
\caption{Convergence of (left) objective function value, (center) constraint value, and (right) dual parameter for simple capacity problem in \eqref{eq_capacity_problem} using proposed DNN method with policy gradients, the exact unparameterized solution, and an equal power allocation amongst users. The DNN parameterization obtains near-optimal performance relative to the exact solution and outperforms the equal power allocation heuristic.}\label{fig_simple_results}
\end{figure*}

\begin{figure}
	\centering
	\includegraphics[width=0.9\linewidth]{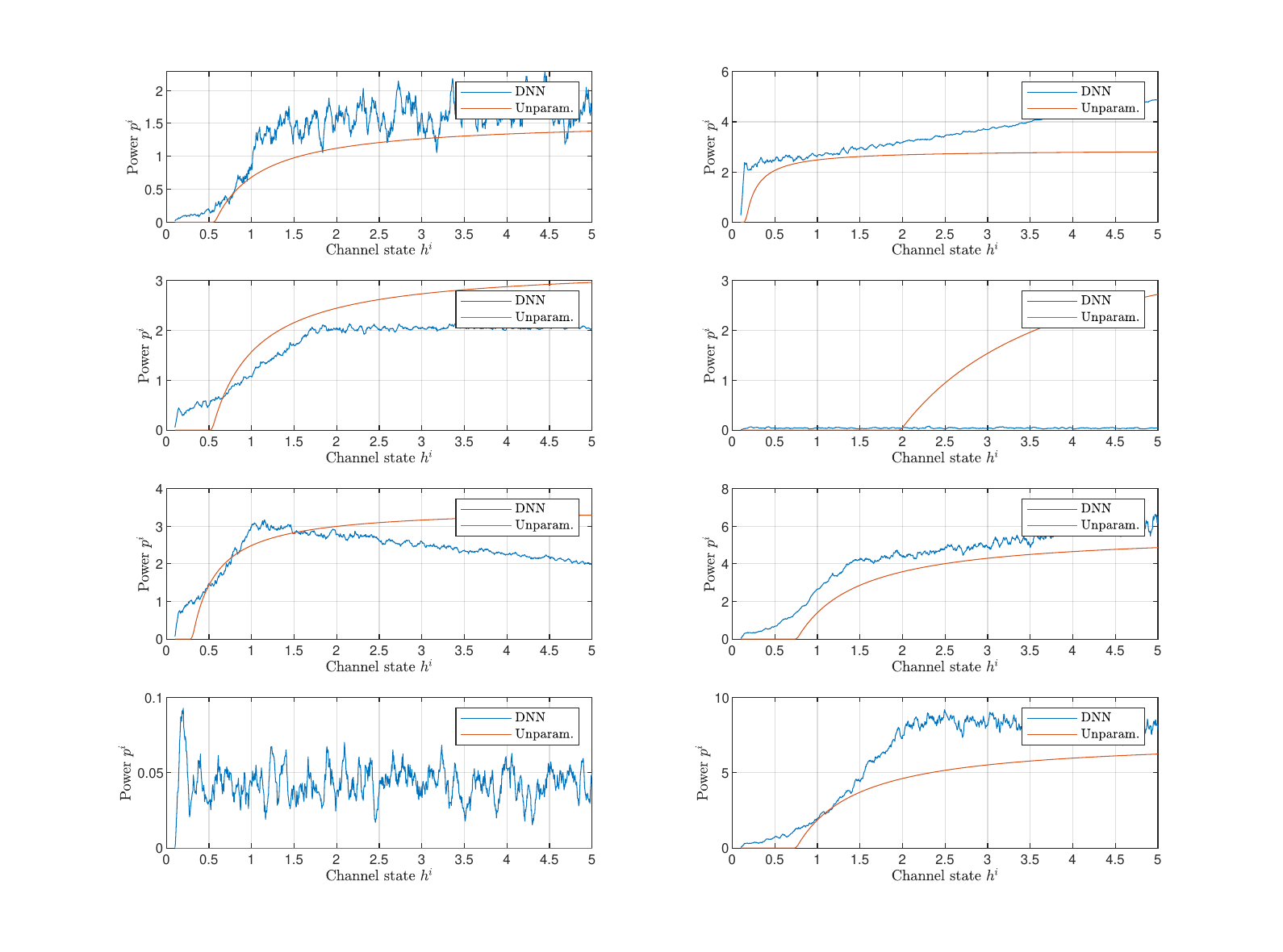} 
	\caption{Example of 8 representative resource allocation policy functions found through DNN parameterization and unparameterized solution. Although the policies differ from the analytic solution, many contain similar shapes. Overall, the DNN method learns variations on the optimal policies that nonetheless achieve similar performance.}
	\label{fig_learned_policies}
\end{figure}

To begin, we simulate the learning of a DNN to solve the problem of maximizing total capacity over a set of simple AWGN wireless fading channel. In this case, each user is given a dedicated channel to communicate, and we wish to allocate resources between users within a total expected power budget $p_{\max}$. In this case, the capacity over the channel can be modeled as $\log (1 + \text{SNR}^i)$, where $\text{SNR}^i := h^i p^i(h^i)/v^i$ is the signal-to-noise ratio experienced by user $i$ and $v^i >0$ is the noise variance. The capacity function for the $i$th user is thus given by $f^i(p^i(h^i),h^i) :=  \log (1 + h^i p^i(h^i)/v^i)$. We are interested in maximizing the weighted aggregate throughput across all users, with user $i$ weighted by $w^i \geq 0$. The total capacity problem can be written as
\begin{align}\label{eq_capacity_problem}
P_{\bbphi}^* &:= \max_{\bbtheta,\bbx} \sum_{i=1}^m w^i x^i \\
&\st \ x^i \leq \E_{h^i} \left[ \log (1 + h^i \phi^i(h^i,\bbtheta)/v^i)\right], \ \forall i\nonumber\\ 
&\qquad \E_{\bbh} \left[ \sum_{i=1}^m \phi^i(h^i,\bbtheta) \right] \leq p_{\max}. \nonumber 
\end{align}
Note that, despite the non-convex structure of the problem in \eqref{eq_capacity_problem}, the loose coupling over the resource allocation variables allows for this problem to be solved exactly without any DNN parametrization using a simple dual stochastic gradient (SGD) method---see, e.g., \cite{wang2010stochastic}. Nonetheless, this is an instructive example with which to validate our approach by seeing if the DNN is capable of learning resource allocation policies that closely match the exact optimal solutions found without any parametrization. Furthermore, the model-free learning capabilities of the DNN parametrization make the proposed learning method applicable in cases in which the, e.g., capacity function is not known.

To have the the outputs of the DNN match the same form as the analytic solution \eqref{eq_capacity_problem}, we construct $m$ independent, uncoupled DNNs for each user. Each channel gain $h^i$ is provided as input to a single-input-single-output (SISO) DNN, which outputs a power allocation $p^i(h^i)$. In particular, each DNN is constructed with two hidden layers, of size 8 and 4, respectively. In addition, each layer is given a ReLU activation function, i.e. $\mathbb{\sigma}(\bbz) = [\bbz]_+$; see Figure \ref{fig_nn_siso} for the architecture. 

The results of a simple experiment with $m=20$ users with random weights $w^i$ and variances $v^i$ is shown in Figure \ref{fig_simple_results}. We further set the maximum power as $p_{\max} = 20$. In this plot we compare the performance of the DNN primal dual learning method with the exact, unparameterized solution and an equal power allocation policy. In the equal power allocation policy, we allocate a power of $p^i = p_{\max}/m$ for all users. Here we see in the left figure that the the total capacity achieved by the DNN primal-dual method converges to roughly the same value as the exact solution found by SGD. Likewise, in the center figure, we plot the value of the constraint function. Here, we see that primal-dual converges to 0, thus implying feasibility of the learned policy. Finally, in the left figure we see that the dual variable obtained by the DNN matches that of the unparameterized. 

\begin{remark}\normalfont
Observe that the learning process may take many iterations to converge than the unparameterized solution due to the many parameters that need to be learned and the model-free nature of the learning process. It is generally the case that the training is done offline before implementation, in which the case the learning rate does not play a significant factor. In the case in which the weights $w^i$ and channel noise power $v^i$ may change over time, we may use the existing as a ``warm-start'' to quickly adapt to the changes in the model. For problem parameters that are changing fast, there have been higher order optimization methods that have been proposed to adapt to changing conditions of the problem \cite{eisen2018learning}. 
\end{remark}

In Figure \ref{fig_learned_policies} we show the actual learned policies from both methods for 8 example users. Here, comparing the optimal unparameterized policies to those learned with DNNs, we see in some cases the policies learned with the DNN match the shape and function, while others differ. For instance, the fourth user shown in Fig \ref{fig_learned_policies} is not assigned any resources by the DNN-based policy, while the seventh user is likewise not given any resources by the unparameterized policy. In any case, the overall performance achieved matches that of the the exact solution. We further note that this phenomenon of certain users not being given any resources in both policies occurs because our only goal in \eqref{eq_capacity_problem} is to maximize the sum-capacity, which does not necessitate that every user gets to transmit. To impose this condition, we may add a constraint to \eqref{eq_capacity_problem} that specifies a maximum average capacity for all users to achieve. 

 %%%%%%%%%%%%%%%%%%%%%%%%%%%%%%%%
%%%%%%%%%% F I G U R E %%%%%%%%%%%%%%%%%
%%%%%%%%%%%%%%%%%%%%%%%%%%%%%%%%
\begin{figure}
\centering
\includegraphics[height=.22\textheight]{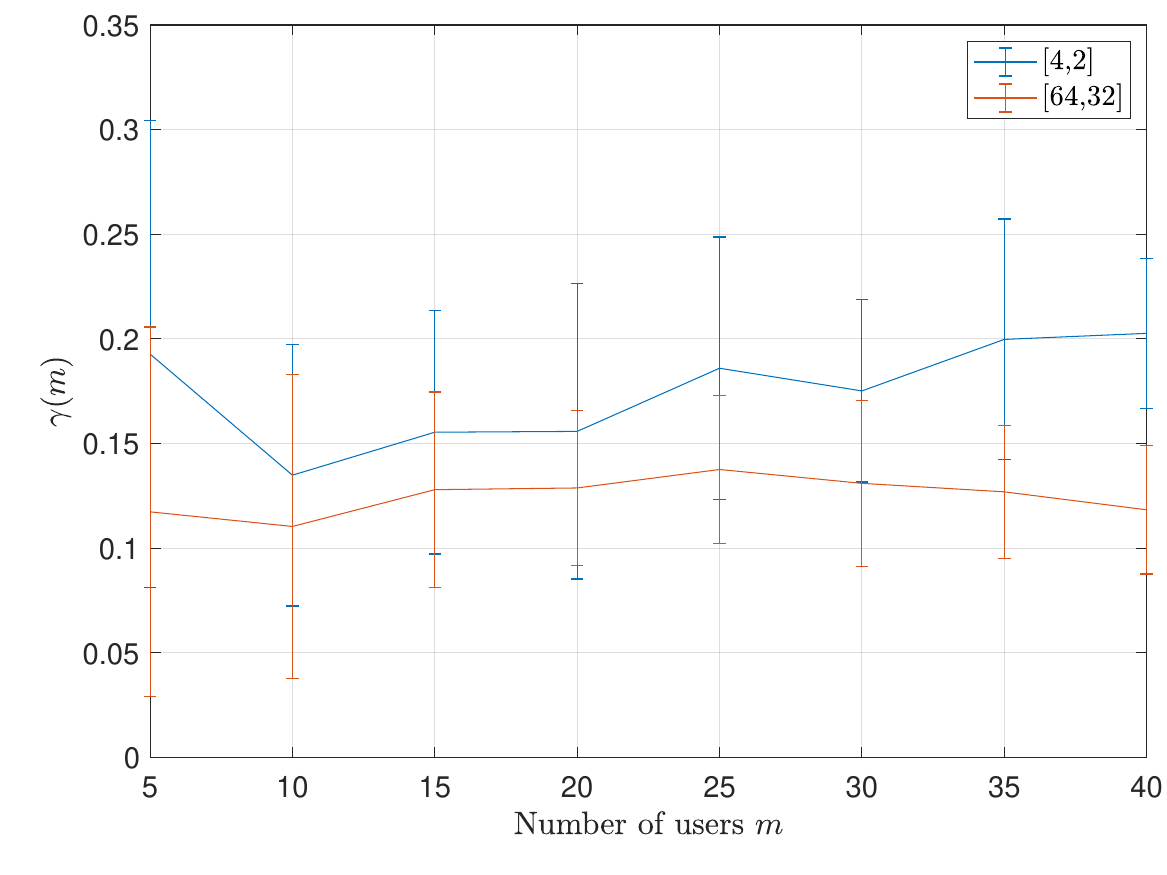} 
\caption{Optimality gap between optimal objective value and learned policy for the simple capacity problem in \eqref{eq_capacity_problem} for different number of users $m$ and DNN architectures. The results are obtained across 10 randomly initialized simulations. The mean is plotted as the solid lines, while the one standard deviation above and below the mean is show with error bars.}
\label{fig_simple_gap}
\end{figure}

For a more thorough comparison of the DNN approach to the exact solution to \eqref{eq_capacity_problem}, we perform multiple experiments for varying number of users and with different DNN layer sizes. In Figure \ref{fig_simple_gap}, we plot the normalized optimality gap between the precise solution $P^*$ and the parameterized solution $\hat{P}_{\bbphi}^*$ found after convergence of the primal-dual method. If we define $P^*(m)$ and $\hat{P}_{\bbphi}^*(m)$ to be the sum capacities achieved by the optimal policy and DNN-based policy found after 40,000 learning iterations, respectively, with $m$ users, the normalized optimality gap can be computed as
\begin{align}
\gamma(m) := \left| \frac{\hat{P}_{\bbphi}^*(m) - P^*(m)}{P^*(m)} \right|.
\end{align}
The blue line shows the results for small DNNs with layer sizes 4 and 2, while the red line shows results for networks with hidden layers of size 32 and 16. Observe that, as the number of channels grows, the DNNs of fixed size achieve the same optimality. Further note that, while the blue line shows that even small DNNs can find near-optimal policies for even large networks, increasing the DNN size increases the expressive power of the DNN, thereby improving upon the suboptimality that can be obtained.

\subsection{Interference channel}
 %%%%%%%%%%%%%%%%%%%%%%%%%%%%%%%%
%%%%%%%%%% F I G U R E %%%%%%%%%%%%%%%%%
%%%%%%%%%%%%%%%%%%%%%%%%%%%%%%%%
%\begin{figure}
%	\centering
%	\includegraphics[width=0.8\linewidth]{wireless_2.png} 
%	\caption{DNN architecture used for Interference channel. The entire vector of channel channel state $h^i$ is fed into a single network for two hidden layers of size 32 and 16, respectively. The DNN outputs a mean $\mu^i$ and standard deviation $\sigma^i$ for a truncated Gaussian distribution for each channel.}
%	\label{fig_nn_mimo}
%\end{figure}

\begin{figure}
\centering
\input{nn_mimo.tex}
\caption{Neural network architecture used for interference channel problem in \eqref{eq_interference_problem0}. All channel states $\bbh$ are fed into a MIMO network with two hidden layers of size 32 and 16, respectively (each circle in hidden layers represents 4 neurons). The DNN outputs means $\mu^i$ and standard deviations $\sigma^i$ for $i$ truncated Gaussian distributions.}
\label{fig_nn_mimo}
\end{figure}
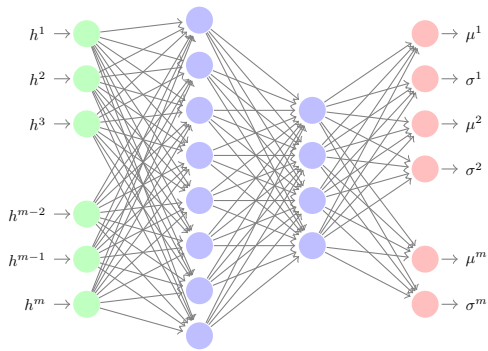

We provide further experiments on the use of neural networks in maximizing capacity over the more complex problem of allocating power over an (IC) interference channel. We first consider the problem of $m$ transmitters communicating with a common receiver, or base station. Given the fading channels $\bbh := [h^1; \hdots; h^m]$, the capacity is determined using the signal-to-noise-plus-interference ratio (SNIR), which for transmission $i$ is given as $\text{SNIR}^i := h^{i} p^i(\bbh))/(v^i + \sum_{j \neq i} h^{j} p^j(\bbh))$. The resulting capacity function observed by the receiver from user $i$ is then given by $f^i(p^i(\bbh),\bbh) :=  \log (1 + h^{ii} p^i(\bbh)/(v^i + \sum_{j \neq i} h^{ji} p^j(\bbh)))$ and the DNN-parameterized problem is written as
\begin{align}\label{eq_interference_problem0}
P_{\bbphi}^* &:= \max_{\bbtheta,\bbx} \sum_{i=1}^m w^i x^i \\
&\st \ x^i \leq \E_{\bbh} \left[ \log \left(1 + \frac{h^{i} \phi^i(\bbh,\bbtheta)}{v^i + \sum_{j \neq i} h^{i} \phi^j(\bbh,\bbtheta)}\right)\right], \ \forall i\nonumber\\ 
&\qquad \E_{\bbh} \left[ \sum_{i=1}^m \phi^i(\bbh,\bbtheta) \right] \leq p_{\max}. \nonumber 
\end{align}
Here, the coupling of the resource policies in the capacity constraint make the problem in \eqref{eq_interference_problem0} very challenging to solve. Existing dual method approaches are ineffective here because the non-convex capacity function cannot be minimized exactly. This makes the primal-dual approach with the DNN parametrization a feasible alternative. However, this means that we cannot provide comparison to the analytic solution, but instead compare against the performance of some standard, model-free heuristic approaches. 

 %%%%%%%%%%%%%%%%%%%%%%%%%%%%%%%%
%%%%%%%%%% F I G U R E %%%%%%%%%%%%%%%%%
%%%%%%%%%%%%%%%%%%%%%%%%%%%%%%%%
\begin{figure*}
\centering
\includegraphics[width=0.4\linewidth, height=.17\textheight, keepaspectratio]{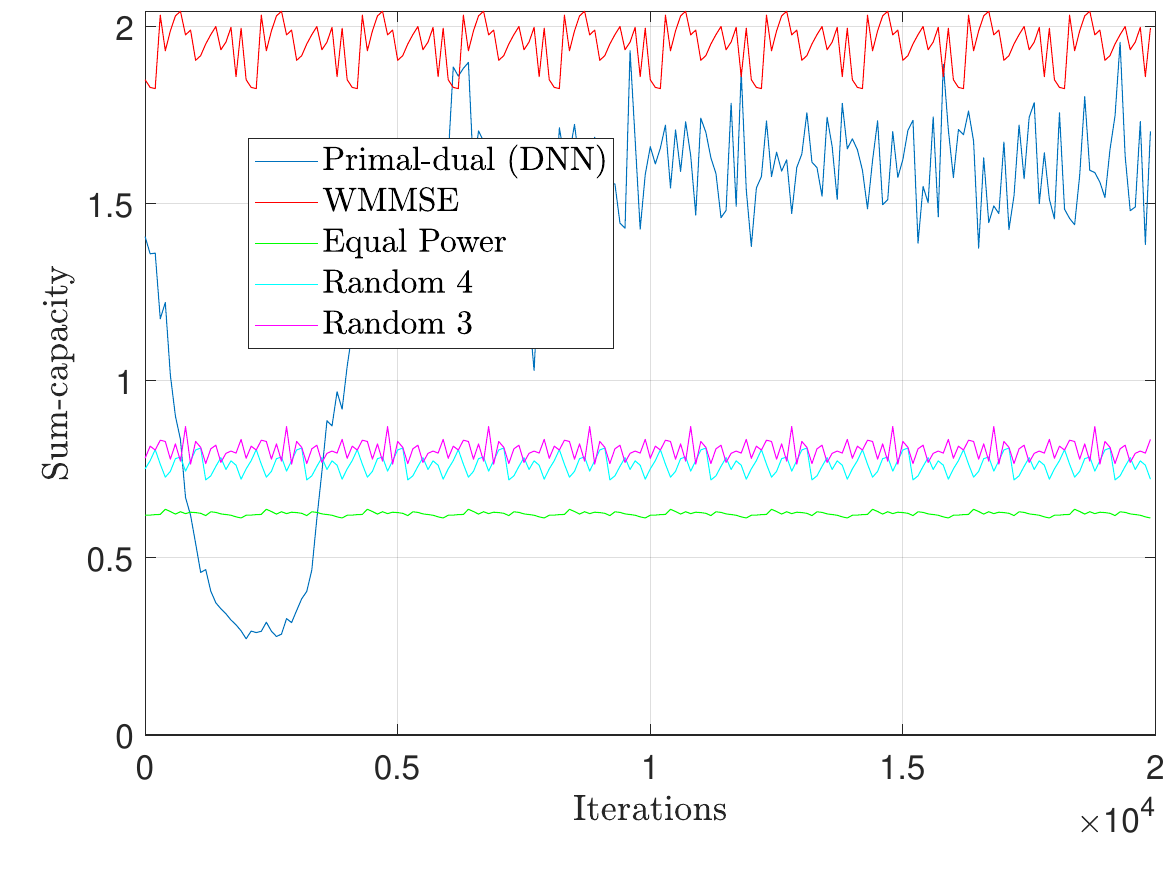} \qquad
\includegraphics[width=0.4\linewidth, height=.17\textheight, keepaspectratio]{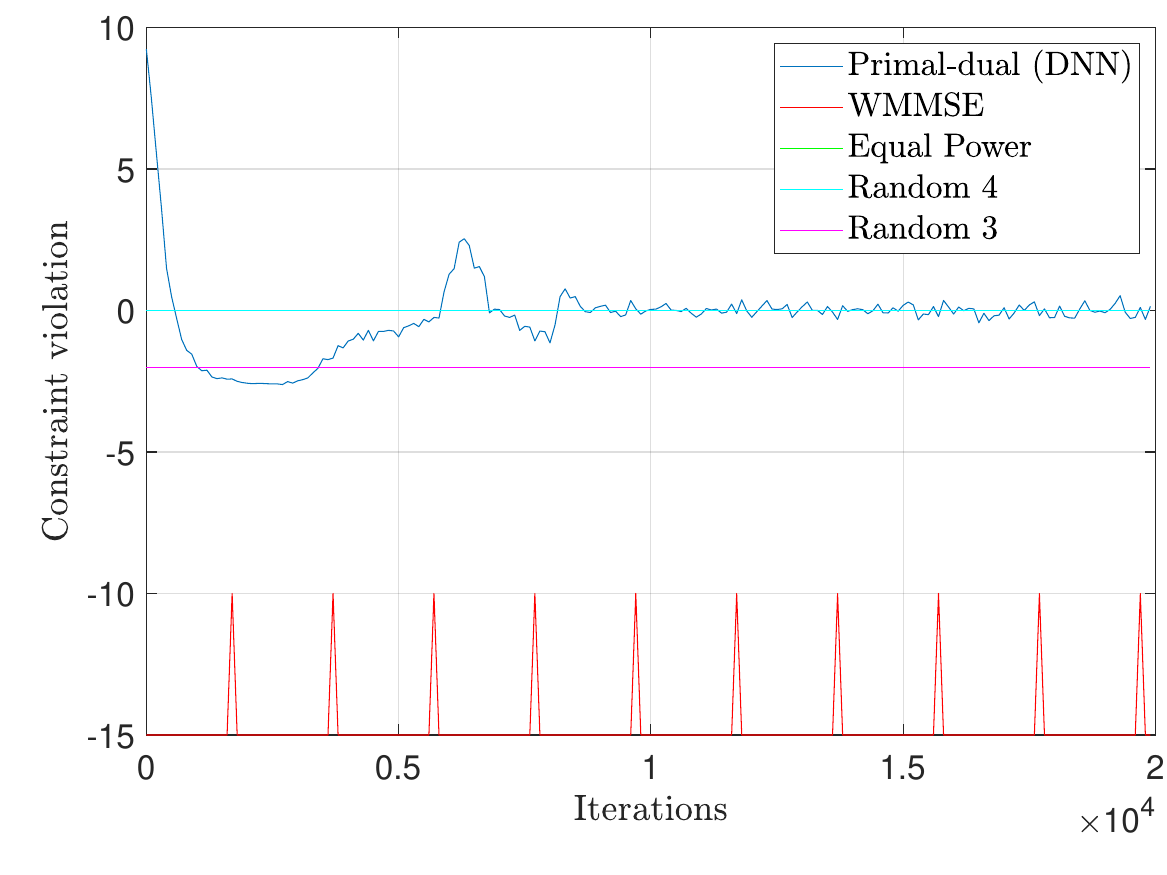} 
\caption{Convergence of (left) objective function value and (right) constraint value for interference capacity problem in \eqref{eq_interference_problem0} using proposed DNN method, WMMSE, and simple model free heuristic power allocation strategies $m=20$ users. The DNN-based primal dual method learns a policy that achieves close performance to WMMSE, better performance than the other model free heuristics, and moreover converges to a feasible solution.} \label{fig_interference_results}
\end{figure*}

As the power allocation of user $i$ will depend on the channel conditions of all users, due to the coupling in the interference channel, rather than the $m$ SISO networks used in the previous example, we construct a single multiple-input-multiple-output (MIMO) DNN architecture, shown in Figure  \ref{fig_nn_mimo}, with a two layers of size 32 and 16 hidden nodes. In this architecture, all channel conditions $\bbh$ are fed as inputs to the DNN, which outputs the truncated Gaussian distribution parameters for every user's policy.  In Figure \ref{fig_interference_results} we plot the convergence of the objective value and constraint value learned using the DNN parameterization and those obtained by three model-free heuristics for a system with $m=20$ users. These include (i) an equal division of power $p_{\max}=20$ across all $m$ users, (ii) randomly selecting 4 users to transmit with power $p = 5$, and (iii) randomly selecting 3 users to transmit with power $p=6$. While not a model free method, we also compare performance against the well-known heuristic method WMMSE \cite{shi2011iteratively}. Here, we observe that, as in the previous example, all values converge to stationary points, suggesting that the method converges to a local optimum. We can also confirm in the right plot of the constraint value that the learned policy is indeed feasible. It can be observed that the performance of the DNN-based policy leaned with the primal dual method is superior to that of the other model free heuristic methods, while obtaining close performance to that of WMMSE, which we stress does indeed require model information to implement.

To study another key aspect of the parameters of the DNN---namely the output distribution $\pi_{\bbtheta,\bbh}$---we make a comparison against the performance achieved using two natural choices for distribution in the power allocation problem in \eqref{eq_interference_problem0}. The primary motivation behind using a truncated Gaussian is that the parameters, namely mean and variance, are easy to interpret and learn. The Gamma distribution, alternatively,  has parameters that are less interpretable in this scenario and the outputs may vary as the parameters change.  In Figure \ref{fig:gamma} we demonstrate the comparison of performance between using a Gamma distribution and the truncated Gaussian distribution. Here, we observe that the performance of the method does indeed rely on proper choice of output distribution, as it can be seen that the truncated Gaussian distribution induces stronger performance relative to a Gamma distribution.
\begin{figure}
\centering
\includegraphics[width=0.6\linewidth]{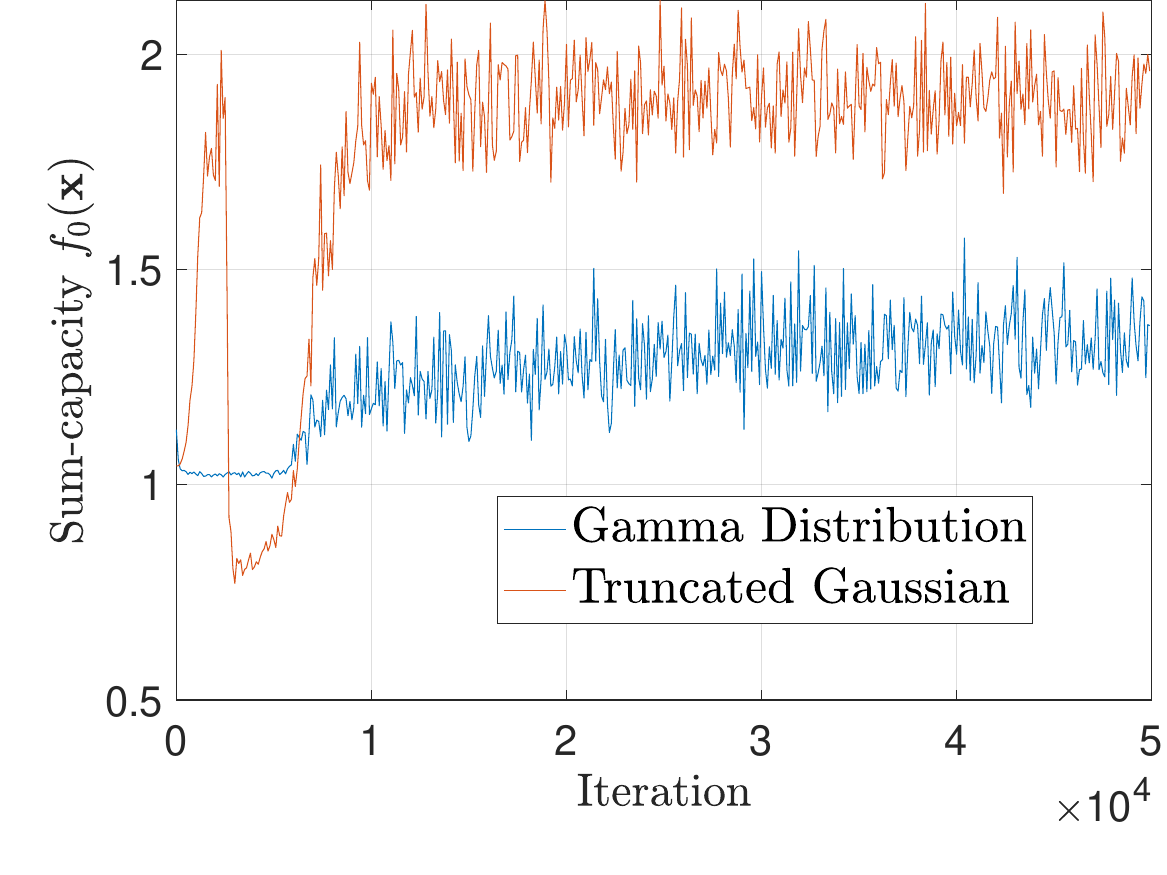}
\caption{Comparison of performance using Gamma and truncated Gaussian distributions in output layer of a DNN.}
\label{fig:gamma}
\end{figure}

Our last series of experiment concerns the classical problem in interference management in which there are $m$ transmitter/receiver pairs sending information to each other. The allocation policy for each user is given as a binary decision $\alpha^i \in \{0,1\}$ of whether or not to transmit with power $p_0$---a variation of this problem is described further detail in Example \ref{example_interference}. In this case, the SNIR for transmission $i$ can be given as $\text{SNIR}^i := h^{ii} p_0 \alpha^i(\bbh))/(v^i + p_0 \sum_{j \neq i} h^{ji} \alpha^j(\bbh))$. The DNN-parameterized problem is written as
\begin{align}\label{eq_interference_problem}
P_{\bbphi}^* &:= \max_{\bbtheta,\bbx} \sum_{i=1}^m w^i x^i \\
&\st \ x^i \leq \E_{\bbh} \left[ \log \left(1 + \frac{h^{ii} p_0 \phi^i(\bbh,\bbtheta)}{v^i + p_0 \sum_{j \neq i} h^{ji} \phi^j(\bbh,\bbtheta)}\right)\right], \ \forall i\nonumber\\ 
&\qquad \E_{\bbh} \left[ \sum_{i=1}^m \phi^i(\bbh,\bbtheta) \right] \leq p_{\max}, \quad \bbphi(\bbh,\bbtheta) \in \{0,1\}^m.\nonumber 
\end{align}
As the case in \eqref{eq_interference_problem0}, this problem cannot be solved exactly. We instead compare the performance against that of the random selection heuristic considered in previous examples.

We plot in Figure \ref{fig_interference_results1} the performance achieved during the learning process for the DNN against the performance of WMMSE and heuristic that randomly selects 2 users to transmit. These simulations are performed on a system of size $m=5$ with a maximum power of $p_{\max} = 20$ and unit weights and variances $w^i=v^i=1$. The DNN has the same fully-connected architecture as used in the previous example. However, given the binary nature of the allocation policies, we employ a Bernoulli distribution as the output policy distribution. In Figure \ref{fig_interference_results1}, we observe that using a DNN learning model, we in fact learn a policy that is close to matching the performance that can be obtained using the WMMSE heuristic and outperforms other model-free heuristics. We further note in the right figure that the learned policy is indeed feasible. This demonstrates the ability of the generic primal-dual learning method to either match or exceed the performance given by heuristic methods that are specifically designed to solve certain problems when applied to problems that do not have a known exact solution. We also stress that the proposed learning method learned such a policy using the model-free learning, thereby not having access to the model for the capacity function, which is necessary in the WMMSE method.

 %%%%%%%%%%%%%%%%%%%%%%%%%%%%%%%%
%%%%%%%%%% F I G U R E %%%%%%%%%%%%%%%%%
%%%%%%%%%%%%%%%%%%%%%%%%%%%%%%%%
\begin{figure*}
\centering
\includegraphics[width=0.4\linewidth, height=.17\textheight, keepaspectratio]{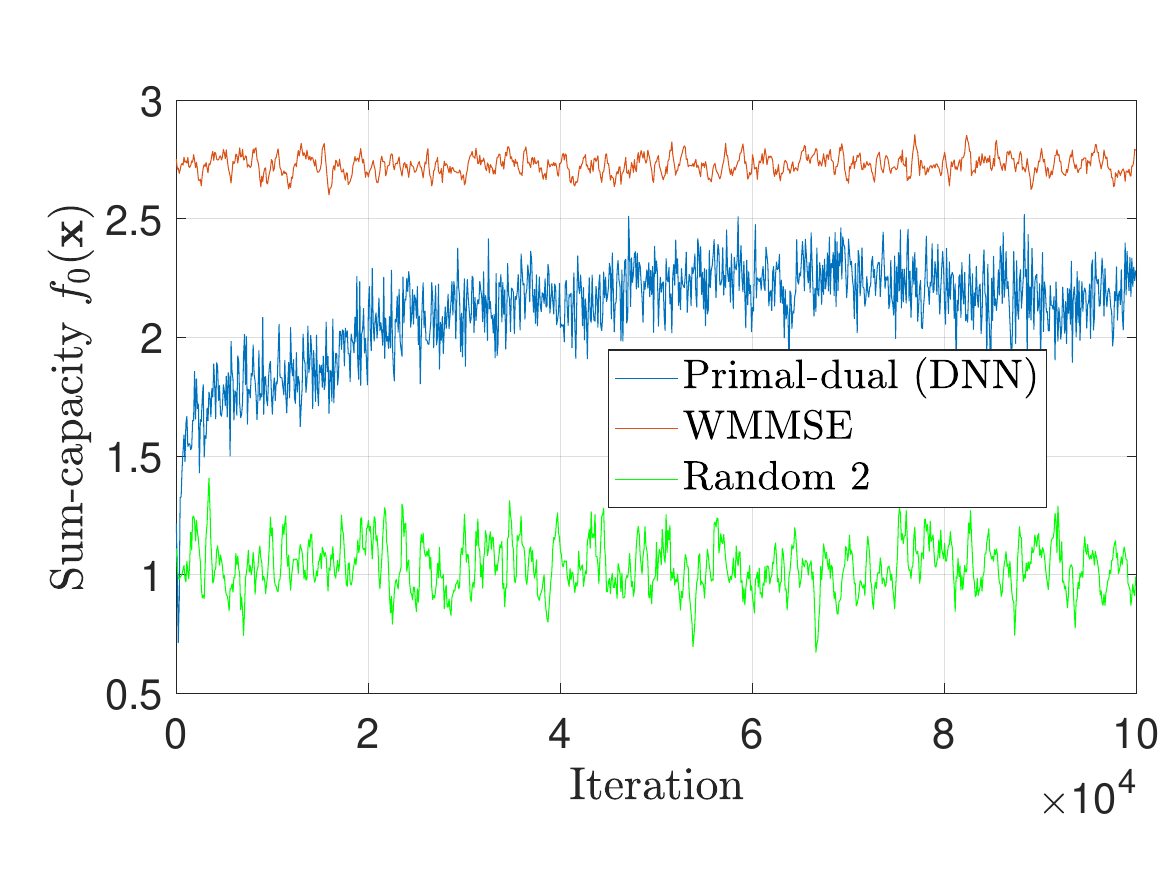} \qquad
\includegraphics[width=0.4\linewidth, height=.17\textheight, keepaspectratio]{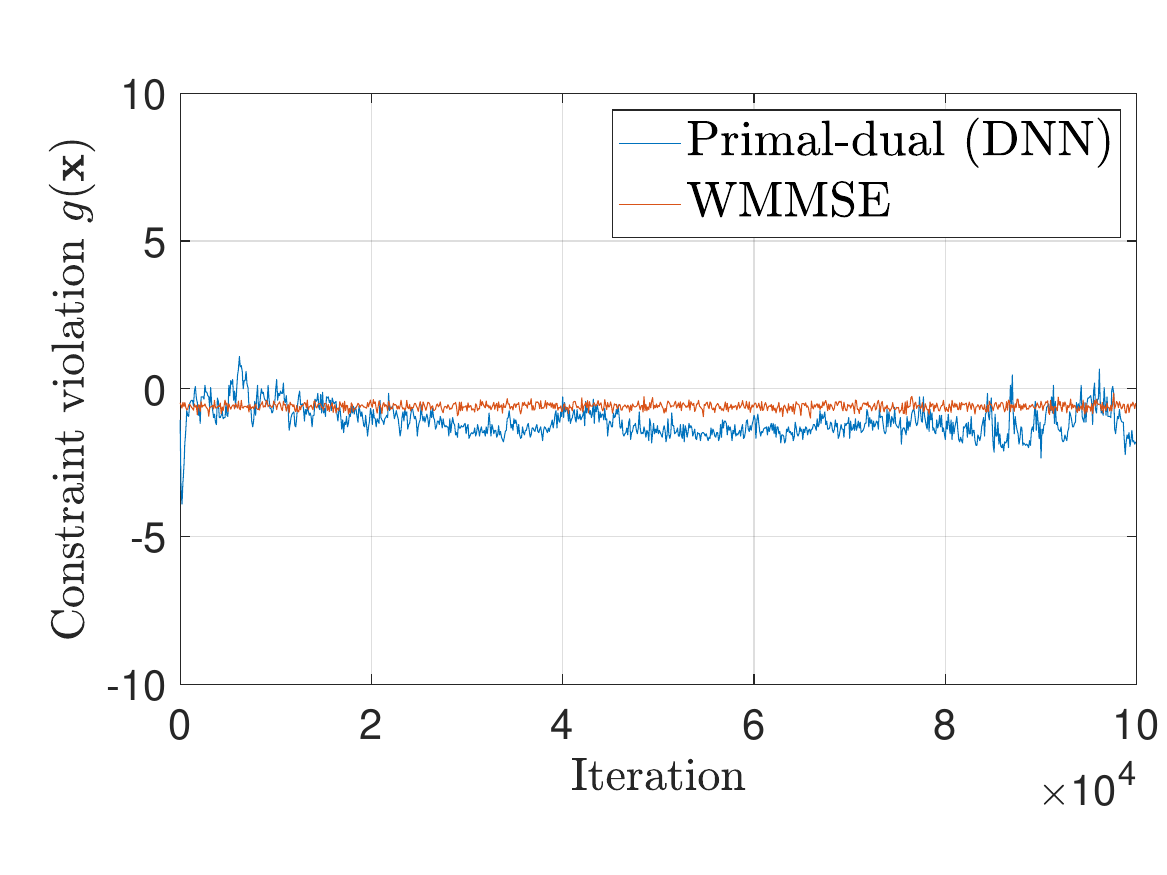} 
\caption{Convergence of (left) objective function value and (right) constraint value for interference capacity problem in \eqref{eq_interference_problem} using proposed DNN method, heuristic WMMSE method, and the equal power allocation heuristic for $m=5$ users. The DNN-based primal dual method learns a policy that is feasible and almost matches the WMMSE method in terms of achieved sum-capacity, without having access to capacity model.} \label{fig_interference_results1}
\end{figure*}

%\red{
%\begin{remark}\label{remark_dnns}\normalfont
%
%It is worth stressing that the the preceding simulation results demonstrating near-optimal performance of DNNs were built using very simple and standard architectures, e.g. 2 hidden layers with 8 and 4 neurons, each. Likewise, the nonlinear activation function used, namely the ReLu, is among the most popularly used activations. Thus, the results obtained were found with minimal tuning of the DNN architecture itself. This is notable because the universal function approximation theorem, and subsequent dual optimality result provided in this paper, rest of the ability to find DNN architectures that can achieve $\eps$ accuracy with respect to the class of functions $\ccalP$ being approximation. This is, in general, not a trivial problem and it is indeed worthwhile to ask how challenging it is to construct architectures that fit this assumption. This is a challenging problem to study analytically due to the complex nature of DNNs, but the simulation results performed here suggest that even simple DNN architectures achieve the necessary approximation properties in such resource allocation problems.  
%\end{remark}
%}

%% file: nn1_siso.tex
\def\layersep{2.5cm}

\begin{tikzpicture}[shorten >=1pt,->,draw=black!50, node distance=.8*\layersep, scale=0.6, transform shape]
    \tikzstyle{every pin edge}=[<-,shorten <=1pt]
    \tikzstyle{neuron}=[circle,fill=black!25,minimum size=17pt,inner sep=0pt]
    \tikzstyle{input neuron}=[neuron, fill=green!25];
    \tikzstyle{output neuron}=[neuron, fill=red!25];
    \tikzstyle{hidden neuron}=[neuron, fill=blue!25];
    \tikzstyle{annot} = [text width=4em, text centered]

    % Draw the input layer nodes
    \foreach \name / \y in {1}
    % This is the same as writing \foreach \name / \y in {1/1,2/2,3/3,4/4}
        \node[input neuron, pin=left: $h^{i}$] (I-\name) at (0,-4) {};
        
% Draw the hidden layer nodes
    \foreach \name / \y in {1,...,8}
        \path[yshift=0.3cm]
            node[hidden neuron] (H1-\name) at (\layersep,-\y cm) {};

    % Draw the hidden layer nodes
    \foreach \name / \y in {1,...,4}
        \path[yshift=0.3cm]
            node[hidden neuron] (H-\name) at (2*\layersep,-2 cm-\y cm) {};

    % Draw the output layer node
    	\node[output neuron,pin={[pin edge={->}]right:$\mu^i$}, right of=H-2] (O1) {};
	\node[output neuron,pin={[pin edge={->}]right:$\sigma^i$}, right of=H-3] (O2) {};

    % Connect every node in the input layer with every node in the
    % hidden layer.
    \foreach \source in {1}
        \foreach \dest in {1,...,8}
            \path (I-\source) edge (H1-\dest);
            
     \foreach \source in {1,...,8}
        \foreach \dest in {1,...,4}
            \path (H1-\source) edge (H-\dest);

    % Connect every node in the hidden layer with the output layer
    \foreach \source in {1,...,4}
        \path (H-\source) edge (O1);
        
     \foreach \source in {1,...,4}
        \path (H-\source) edge (O2);

    % Annotate the layers
  %  \node[annot,above of=H-1, node distance=1cm] (hl) {Hidden layer};
  %  \node[annot,left of=hl] {Input layer};
   % \node[annot,right of=hl] {Output layer};
\end{tikzpicture}

% End of code

%% file: nn_mimo.tex
\def\layersep{2.5cm}

\begin{tikzpicture}[shorten >=1pt,->,draw=black!50, node distance=.8*\layersep, scale=0.6, transform shape]
    \tikzstyle{every pin edge}=[<-,shorten <=1pt]
    \tikzstyle{neuron}=[circle,fill=black!25,minimum size=17pt,inner sep=0pt]
    \tikzstyle{tneuron}=[circle,line width=1.2pt,minimum size=17pt,inner sep=0pt]
    \tikzstyle{input neuron}=[neuron, fill=green!25];
    \tikzstyle{output neuron}=[neuron, fill=red!25];
    \tikzstyle{hidden neuron}=[tneuron, fill=blue!25];
    \tikzstyle{annot} = [text width=4em, text centered]

    % Draw the input layer nodes
  %  \foreach \name / \y in {1,...,6}
    % This is the same as writing \foreach \name / \y in {1/1,2/2,3/3,4/4}
        \node[input neuron, pin=left: $h^{1}$] (I-1) at (0,-1) {};
        \node[input neuron, pin=left: $h^{2}$] (I-2) at (0,-2) {};
        \node[input neuron, pin=left: $h^{3}$] (I-3) at (0,-3) {};
        
        \node[input neuron, pin=left: $h^{m-2}$] (I-4) at (0,-5) {};
        \node[input neuron, pin=left: $h^{m-1}$] (I-5) at (0,-6) {};
        \node[input neuron, pin=left: $h^{m}$] (I-6) at (0,-7) {};
        
% Draw the hidden layer nodes
    \foreach \name / \y in {1,2,3,4,5,6,7,8}
        \path[yshift=0.3cm]
            node[hidden neuron] (H1-\name) at (\layersep,-\y cm) {};

    % Draw the hidden layer nodes
    \foreach \name / \y in {1,...,4}
        \path[yshift=0.3cm]
            node[hidden neuron] (H-\name) at (2*\layersep,-2 cm-\y cm) {};

    % Draw the output layer node
    	\node[output neuron,pin={[pin edge={->}]right:$\mu^1$}] (O1) at (3*\layersep,-1 cm) {};
	\node[output neuron,pin={[pin edge={->}]right:$\sigma^1$}] (O2) at (3*\layersep,-2 cm) {};
	\node[output neuron,pin={[pin edge={->}]right:$\mu^2$}] (O3) at (3*\layersep,-3 cm) {};
	\node[output neuron,pin={[pin edge={->}]right:$\sigma^2$}] (O4) at (3*\layersep,-4 cm) {};
	\node[output neuron,pin={[pin edge={->}]right:$\mu^m$}] (O5) at (3*\layersep,-6 cm) {};
	\node[output neuron,pin={[pin edge={->}]right:$\sigma^m$}] (O6) at (3*\layersep,-7 cm) {};
	%\node[output neuron,pin={[pin edge={->}]right:$\sigma^i$}, right of=H-3] (O2) {};

    % Connect every node in the input layer with every node in the
    % hidden layer.
    \foreach \source in {1,...,6}
        \foreach \dest in {1,...,8}
            \path (I-\source) edge (H1-\dest);
            
     \foreach \source in {1,...,8}
        \foreach \dest in {1,...,4}
            \path (H1-\source) edge (H-\dest);

    % Connect every node in the hidden layer with the output layer
    \foreach \source in {1,...,4}
        \path (H-\source) edge (O1);
        
     \foreach \source in {1,...,4}
        \path (H-\source) edge (O2);
        
            \foreach \source in {1,...,4}
        \path (H-\source) edge (O3);
        
     \foreach \source in {1,...,4}
        \path (H-\source) edge (O4);
        
            \foreach \source in {1,...,4}
        \path (H-\source) edge (O5);
        
     \foreach \source in {1,...,4}
        \path (H-\source) edge (O6);

    % Annotate the layers
  %  \node[annot,above of=H-1, node distance=1cm] (hl) {Hidden layer};
  %  \node[annot,left of=hl] {Input layer};
   % \node[annot,right of=hl] {Output layer};
\end{tikzpicture}

% End of code

%% file: duality_proof.tex
%!TEX root = root.tex

To inform the analysis of the suboptimality of $D_{\bbphi}^*$ from \eqref{eq_param_dual0}, we first present an established result previously referenced, namely the null duality gap property of the original problem in \eqref{eq_param_problem}. We proceed by presenting the associated Lagrangian function and dual problem for the constrained optimization problem in \eqref{eq_problem}:
\begin{align}\label{eq_lagrangian}
\ccalL(\bbp(\bbh),\bbx, \bbmu, \bblambda) &:= g_0(\bbx) + \bbmu^T \bbg(\bbx) \\& \qquad + \bblambda^T \left(  \E_{\bbh} \left[ \bbf(\bbp(\bbh),\bbh)\right] - \bbx \right), \nonumber \\
D^* &:= \min_{\bblambda,\bbmu \geq \bb0} \max_{\bbp \in \ccalP, \bbx \in \ccalX} \ccalL(\bbp(\bbh),\bbx, \bbmu, \bblambda). \label{eq_dual}
\end{align}
Despite non-convexity of \eqref{eq_problem}, a known result established in \cite{ribeiro2012optimal} demonstrates that problems of this form indeed satisfy a null duality gap property given the technical conditions previously presented. Due to the central role it plays in the proceeding analysis of \eqref{eq_param_dual0}, we present this theorem here for reference.
\begin{theorem}\cite[Theorem 1]{ribeiro2012optimal}\label{theorem_null_duality}
Consider the optimization problem in \eqref{eq_problem} and its Lagrangian dual in \eqref{eq_dual}. Provided that Assumptions \ref{assumption_nonatomic} and \ref{assumption_slater} hold, then the problem in \eqref{eq_problem} exhibits null duality gap, i.e., $P^* = D^*$.
\end{theorem}

With this result in mind, we begin to establish the result in \eqref{eq_theorem_param_duality} by considering the upper bound. First, note that the dual problem of~\eqref{eq_param_problem} defined in~\eqref{eq_param_dual0} can be written as
\begin{multline}\label{eq_param_dual2}
	D_{\bbphi}^* = \min_{\bblambda,\bbmu\geq \bb0}
		\left\{ \max_{\bbx \in \ccalX} g_0(\bbx) + \bbmu^T \bbg(\bbx) - \bblambda^T\bbx \right.
		\\
		{} \left.+ \max_{\bbtheta \in \Theta}
			\bblambda^T\E_{\bbh} \left[ \bbf(\bbphi(\bbh,\bbtheta),\bbh)\right] \right\}
			\text{.}
\end{multline}
Focusing on the second term, observe then that for any solution~$\bbp^*(\bbh)$ of~\eqref{eq_problem}, since~$\ccalP_{\bbphi} \subseteq \ccalP$, it holds that
%%
%\begin{multline}\label{eq_param_dual3}
%	\max_{\bbtheta \in \Theta} \bblambda^T
%		\E_{\bbh} \left[ \bbf(\bbphi(\bbh,\bbtheta),\bbh) \right]
%		= \bblambda^T \E_{\bbh} \left[\bbf(\bbp^*(\bbh),\bbh)\right]
%	\\
%	{}+ \max_{\bbtheta \in \Theta} \bblambda^T
%		\E_{\bbh} \left[ \bbf( \bbphi(\bbh,\bbtheta),\bbh) - \bbf( \bbp^*(\bbh),\bbh) \right]
%		\text{.}
%\end{multline}
%%
%Since~$\ccalP_{\bbphi} \subseteq \ccalP$, it must be that~$g_0(\bbx_\theta^\star) \leq g_0(\bbx^\star)$, where~$\bbx^\star$ and~$\bbx_\theta^\star$ are the maximizers of~\eqref{eq_problem} and~\eqref{eq_param_problem} respectively. Because~$g_0$ is monotonically non-decreasing, the ergodic constraint holds with equality and~$g_0(\bbx_\theta^\star) \leq g_0(\bbx^\star)$ implies that
%%
%\begin{equation*}
%	\E_{\bbh} \left[ \bbf(\bbh,\bbphi(\bbh,\bbtheta^\star)) \right] =
%		\bbx_\theta^\star \leq \bbx^\star
%		= \E_{\bbh} \left[ \bbf(\bbp^*(\bbh),\bbh) \right]
%		\text{,}
%\end{equation*}
%%
%where~$\bbtheta^\star$ is a solution of~\eqref{eq_param_problem}. By optimality, it holds that~$\E_{\bbh} \left[ \bbf(\bbphi(\bbh,\bbtheta),\bbh) - \bbf(\bbp^*(\bbh),\bbh) \right] \leq 0$ for all~$\bbtheta \in \Theta$. Since~$\bblambda \geq \bb0$, \eqref{eq_param_dual3} yields
%
\begin{equation}\label{eq_param_dual4}
	\max_{\bbtheta \in \Theta} \bblambda^T
		\E_{\bbh} \left[ \bbf(\bbphi(\bbh,\bbtheta),\bbh) \right]
	\leq \bblambda^T \E_{\bbh} \left[\bbf(\bbp^*(\bbh),\bbh)\right]
		\text{.}
\end{equation}
Substituting~\eqref{eq_param_dual4} back into~\eqref{eq_param_dual2} and using the strong duality result from Theorem~\ref{theorem_null_duality}, we obtain
\begin{multline}\label{eq_param_dual5}
	D_{\bbphi}^* \leq \min_{\lambda,\bbmu \geq \bb0}
		\max_{\bbx \in \ccalX} g_0(\bbx) + \bbmu^T \bbg(\bbx) - \bblambda^T\bbx
	\\
	{}+ \bblambda^T \E_{\bbh} \left[ \bbf(\bbp^*(\bbh),\bbh) \right] = D^* = P^*
		\text{,}
\end{multline}
where we used the fact that the right-hand side of the inequality in~\eqref{eq_param_dual5} is the optimal dual value of problem~\eqref{eq_problem} as defined in~\eqref{eq_dual}.

We prove the lower bound in~\eqref{eq_theorem_param_duality} by proceeding in a similar manner, i.e., by manipulating the expression of the dual value in~\eqref{eq_param_dual2}. In contrast to the previous bound, however, we obtain a perturbed version of~\eqref{eq_problem} which leads to the desired bound. Explicitly, notice that for all~$\bbp \in \ccalP$ it holds that
\begin{multline}\label{eq_param_dual1b}
	\max_{\bbtheta \in \Theta} \bblambda^T
		\E_{\bbh} \left[ \bbf(\bbphi(\bbh,\bbtheta),\bbh) \right]
	= \bblambda^T \E_{\bbh} \left[\bbf(\bbp(\bbh),\bbh) \right]
	\\
	{}- \min_{\bbtheta \in \Theta} \bblambda^T
		\E_{\bbh} \left[ \bbf( \bbp(\bbh),\bbh) - \bbf( \bbphi(\bbh,\bbtheta),\bbh) \right]
		\text{,}
\end{multline}
where we used the fact that for any~$f_0$ and~$\ccalY$, it holds that~$\max_{y \in \ccalY} f_0(y) = -\min_{y \in \ccalY} -f_0(y)$. Then, apply H\"older's inequality to bound the second term in~\eqref{eq_param_dual1b} as
\begin{multline}\label{eq_param_dual2b}
	\min_{\bbtheta \in \Theta} \bblambda^T
		\E_{\bbh} \left[ \bbf( \bbp(\bbh),\bbh) - \bbf( \bbphi(\bbh,\bbtheta),\bbh) \right]
	\\
	{}\leq \|\bblambda\|_{1} \left[ \min_{\bbtheta \in \Theta}
		\|\E_{\bbh} \left[ \bbf( \bbp(\bbh),\bbh)
			- \bbf( \bbphi(\bbh,\bbtheta),\bbh) \right]\|_{\infty} \right]
		\text{.}
\end{multline}
To upper bound the minimization in~\eqref{eq_param_dual2b}, start by using the convexity of the infinity norm and the continuity of~$\E_{\bbh}\bbf(\bbh, \cdot)$ to obtain
\begin{multline*}
	\min_{\bbtheta \in \Theta} \|\E_{\bbh} \left[ \bbf( \bbp(\bbh),\bbh)
			- \bbf( \bbphi(\bbh,\bbtheta),\bbh) \right]\|_{\infty}
	\\
	{}\leq \min_{\bbtheta \in \Theta} \E_{\bbh} \left[ \| \bbf( \bbp(\bbh),\bbh)
			- \bbf( \bbphi(\bbh,\bbtheta),\bbh) \|_{\infty} \right]
	\\
	{}\leq \min_{\bbtheta \in \Theta}
		\E_{\bbh} \left[ L \| \bbp(\bbh) - \bbphi(\bbh,\bbtheta) \|_{\infty} \right]
		\text{.}
\end{multline*}
The definition in~\eqref{eq_def_bound} then readily gives
\begin{equation}\label{eq_param_dual3b}
	\min_{\bbtheta \in \Theta} \|\E_{\bbh} \left[ \bbf( \bbp(\bbh),\bbh)
			- \bbf( \bbphi(\bbh,\bbtheta),\bbh) \right]\|_{\infty}
		\leq L \epsilon
		\text{.}
\end{equation}
Substituting~\eqref{eq_param_dual2b} and~\eqref{eq_param_dual3b} into~\eqref{eq_param_dual1b} yields
\begin{equation*}
	\max_{\bbtheta \in \Theta} \bblambda^T
		\E_{\bbh} \left[ \bbf(\bbphi(\bbh,\bbtheta),\bbh) \right]
	\geq \bblambda^T \E_{\bbh} \left[\bbf(\bbp(\bbh),\bbh) \right]
	- \|\bblambda\|_{1} L \epsilon
		\text{,}
\end{equation*}
which we can then use in the definition of the dual value~\eqref{eq_param_dual2} to obtain
\begin{multline}\label{eq_param_dual4b}
	D_{\bbphi}^* \geq \min_{\bblambda,\bbmu\geq \bb0}
		\max_{\bbx \in \ccalX} g_0(\bbx) + \bbmu^T \bbg(\bbx) - \bblambda^T\bbx
		\\
		{}+ \bblambda^T \E_{\bbh} \left[\bbf(\bbp(\bbh),\bbh) \right]
	- \|\bblambda\|_{1} L \epsilon
			\text{.}
\end{multline}

We are now ready to derive the perturbed version of~\eqref{eq_problem} in order to obtain our lower bound. To do so, notice that~$\bblambda \geq 0$ implies that~$\| \bblambda \|_1 = \bblambda^T \mathbf{1}$, where~$\mathbf{1}$ is a column vector of ones. Since~\eqref{eq_param_dual4b} holds for all~$\bbp \in \ccalP$, we get
\begin{multline}\label{eq_param_dual5b}
	D_{\bbphi}^* \geq \min_{\bblambda,\bbmu\geq \bb0}
		\max_{\bbx \in \ccalX} g_0(\bbx) + \bbmu^T \bbg(\bbx) - \bblambda^T\bbx
		\\
		{}+ \max_{\bbp \in \ccalP} \bblambda^T
			\left\{ \E_{\bbh} \left[\bbf(\bbp(\bbh),\bbh) \right]
			- L \epsilon \mathbf{1} \right\}
			\text{.}
\end{multline}
Now, observe that the right-hand side of~\eqref{eq_param_dual5b} is the dual value of an~$(L\epsilon)$-perturbed version of~\eqref{eq_problem}
\begin{align}\label{eq_problem_b}
	P_{L\epsilon}^* := &\max_{\bbp,\bbx} C (\bbx)
	\nonumber\\
	&\st \ L \epsilon \mathbf{1} + \bbx \leq \E_{\bbh} \left[ \bbf(\bbp(\bbh),\bbh)\right],
		\quad \bb0 \leq \bbg(\bbx),
	\nonumber \\ 
	&\quad \quad \ \bbx \in \ccalX, \quad \bbp \in \ccalP
\end{align}
Naturally, \eqref{eq_problem_b} has the same strong duality property as~\eqref{eq_problem} from Theorem~\ref{theorem_null_duality}, which implies that~$D_{\bbphi}^* \geq D_{L\epsilon}^* = P_{L\epsilon}^*$. A well-known perturbation inequalitiy, e.g., \cite[Eq.~(5.57)]{boyd2004convex}, relates $P_{L\epsilon}^*$ to $P^*$ as
\begin{align}\label{eq_pllol}
P_{L\epsilon}^* \geq P^* -  \|\bblambda^*\|_1 L \eps.
\end{align}
Combining \eqref{eq_pllol} with $D_{\bbphi}^* \geq P_{L\epsilon}^*$, we obtain \eqref{eq_theorem_param_duality}.
%} \blue{Explain. Show the equations.}

We proceed to prove the bound in \eqref{eq_lambda_norm_bound}. Note that the strong duality result in Theorem~\ref{theorem_null_duality} implies that
\begin{equation}\label{E:preDualVarBound}
\begin{aligned}
	P^* = D^* &=
		\max_{\bbp \in \ccalP, \bbx \in \ccalX}
		g_0(\bbx) + {\bbmu^*}^T \bbg(\bbx)
	\\
	{}&+ {\bblambda^*}^T \left(
			\E_{\bbh} \left[ \bbf(\bbp(\bbh),\bbh) \right] - \bbx
		\right)
	\\
	{}&\geq g_0(\bbx^\prime) + {\bbmu^*}^T \bbg(\bbx^\prime)
	\\
	{}&+ {\bblambda^*}^T \left(
			\E_{\bbh} \left[ \bbf(\bbh,\bbp^\prime(\bbh)) \right]
			- \bbx^\prime
		\right)
		\text{,}
\end{aligned}
\end{equation}
where~$(\bblambda^*, \bbmu^*)$ are the minimizers of~\eqref{eq_dual} and~$(\bbx^\prime, \bbp^\prime)$ are arbitrary feasible points of~\eqref{eq_problem}. Since Slater's condition holds, we can choose~$(\bbx^\prime, \bbp^\prime)$ to be strictly feasible, i.e., such that~$\bbg(\bbx^\prime) > \bb0$ and~$\E_{\bbh} \left[ \bbf(\bbh,\bbp^\prime(\bbh)) \right] > \bbx^\prime$, to obtain
\begin{align}
	P^* &\geq g_0(\bbx^\prime) + {\bblambda^*}^T \left(
			\E_{\bbh} \left[ \bbf(\bbh,\bbp^\prime(\bbh)) \right]
			- \bbx^\prime
		\right)
	\notag\\
	{}&\geq g_0(\bbx^\prime) + {\bblambda^*}^T \mathbf{1} \cdot s
		\text{,}
		\label{E:preDualVarBound2}
\end{align}
where~$s = \min_i \E_{\bbh} \left[ f_i(\bbh,\bbp^\prime(\bbh)) \right] - x_i^\prime$ as defined in Assumption \ref{assumption_slater}, with~$\bbf = [f_i]$ and~$\bbx^\prime = [x_i]$. Note that~$s > 0$ since~$\bbx^\prime$ is strictly feasible. Finally, since~$\bblambda^* \geq 0$, we can rearrange~\eqref{E:preDualVarBound2} to obtain
\begin{equation}
	\norm{\bblambda^*}_1 \leq \frac{P^* - g_0(\bbx^\prime)}{s}
		\text{.}
\end{equation}

\section{Proof of Theorem \ref{cor_param_duality}}\label{sec_cor_proof}

We start by presenting the well-established result that a DNN of arbitrarily large size is a universal parameterization for measurable functions in probability.

\begin{theorem}\cite[Theorem 2.2]{hornik1991approximation}\label{theorem_approx}
Define~$\ccalD = \{\bbphi(\cdot,\bbtheta): \bbtheta \in \reals^q\}$ to be the set of all functions described by the DNN in \eqref{eq_policy_dnn} with $\bm{\sigma}_{l}$ non-constant and continuous for all $l = 1,\dots,L$. Then, for an arbitrarily large number~$q$ of hidden nodes, $\ccalD$ is dense in probability in the set of measurable functions $\ccalM$, i.e., for every function $\hbp(\bbh) \in \ccalM$ and all $\tilde{\eps} > 0$, there exists a $q > 0$ and $\bbtheta \in \reals^q$ such that
\begin{align}\label{eq_lol1}
	m\left( \{ \bbh \in \ccalH :
		\| \hbp(\bbh) - \bbphi(\bbh,\bbtheta) \|_{\infty} > \tilde{\eps} \} \right)
		< \tilde{\eps}.
\end{align}
\end{theorem}

Using Theorem \ref{theorem_approx}, we can show that DNNs satisfy the~$\eps$-universality condition from Definition~\ref{def_universal} for all~$\eps > 0$.

\begin{lemma}\label{lemma_eps_univeral}
The entire class of DNN parameterizations $\bbphi \in \Phi$, where $\bbphi$ is defined in \eqref{eq_policy_dnn} with non-constant, continuous activation~$\bm{\sigma}_{l}$ for some layer size $\bbq > \bb0$, is an~$\eps$-universal parametrization as in Definition~\ref{def_universal} for all~$\eps > 0$.
\end{lemma}

\begin{proof}

Let~$\ccalK_{\eps^\prime} = \{ \bbh \in \ccalH : \| \hbp(\bbh) - \bbphi(\bbh,\bbtheta) \|_{\infty} > \eps^\prime \} \subseteq \ccalH$ and observe that~\eqref{eq_def_bound} can be written as
\begin{align*}
	\E \left\| \bbp(\bbh) - \bbphi(\bbh,\bbtheta) \right\|_{\infty} &=
		\int_{\ccalH \setminus \ccalK_{\eps^\prime}}
			\left\| \bbp(\bbh) - \bbphi(\bbh,\bbtheta) \right\|_{\infty} dm(\bbh)
	\\
	{}&+ \int_{\ccalK_{\eps^\prime}}
		\left\| \bbp(\bbh) - \bbphi(\bbh,\bbtheta) \right\|_{\infty} dm(\bbh)
		\text{,}
\end{align*}
where~$m$ is the probability measure from which the channel state is drawn. It is ready that the first integral is upper bounded by~$\eps^\prime \cdot m(\ccalH \setminus \ccalK_{\eps^\prime}) < \eps^\prime \cdot m(\ccalH) < \eps^\prime$ for all~$\eps^\prime > 0$. To bound the second integral, recall that the set of feasible policies~$\ccalP$ is bounded and let~$\Gamma = \sup\{\|\hbp(\bbh)\|_{\infty} : \hbp \in \ccalP \text{ and } \bbh \in \ccalH\} < \infty$. Then, from~\eqref{eq_lol1}, we obtain the following bound over all DNNs $\bbphi$, i.e.
\begin{equation*}
	\inf_{\bbphi \in \Phi} \int_{\ccalK_{\eps^\prime}}
		\left\| \bbp(\bbh) - \bbphi(\bbh,\bbtheta) \right\|_{\infty} dm(\bbh)
		< 2 \Gamma \cdot m(\ccalK_{\eps^\prime}) < 2 \Gamma \eps^\prime.
\end{equation*}
Thus, for all~$\eps^\prime > 0$,
\begin{equation}\label{eq_dnn_bound}
	\inf_{\bbphi \in \Phi} \E \left\| \bbp(\bbh) - \bbphi(\bbh,\bbtheta) \right\|_{\infty} < (1 + 2\Gamma) \eps^\prime.
\end{equation}
Taking~$\eps^\prime = \eps/(1+2\Gamma)$ in~\eqref{eq_dnn_bound} yields~\eqref{eq_def_bound}.
\end{proof}

Lemma~\ref{lemma_eps_univeral} implies that the dual value bound~\eqref{eq_theorem_param_duality} from Theorem~\ref{theorem_param_duality} holds for all~$\epsilon > 0$ if we consider the entire class of DNN functions $\bbphi \in \Phi$. Since the Lipschitz constant~$L < \infty$, the only obstacle to completing a continuity argument is if~$\norm{\bblambda^*}_1$ is unbounded. However, recall from \eqref{eq_lambda_norm_bound} that
\begin{equation}
	\norm{\bblambda^*}_1 \leq \frac{P^* - g_0(\bbx_0)}{s}
		< \infty
		\text{.}
\end{equation}
Hence, we obtain that
\begin{equation*}
	P^* - \delta \leq \inf_{\bbphi \in \Phi} D^*_{\bbphi} \leq P^*
\end{equation*}
for all~$\delta > 0$~(simply take~$\epsilon = \delta \norm{\bblambda^*}_1^{-1} L^{-1} > 0$). Then, there would exist~$\delta^\prime > 0$ such that~$P^* > D^*_{\bbphi} + \delta^\prime$~(e.g., take $\delta^\prime$ to be the midpoint between~$P^*$ and~$D^*_{\bbphi}$), which would contradict Theorem~\ref{theorem_param_duality}. Hence, $\inf_{\bbphi} D^*_{\bbphi} = P^*$.

%By continuity, it must be that~$D^*_{\bbphi} = P^*$. Indeed, suppose that~$D^*_{\bbphi} < P^*$. Then, there would exist~$\delta^\prime > 0$ such that~$P^* > D^*_{\bbphi} + \delta^\prime$~(e.g., take $\delta^\prime$ to be the midpoint between~$P^*$ and~$D^*_{\bbphi}$), which would contradict Theorem~\ref{theorem_param_duality}. Hence, $D^*_{\bbphi} = P^*$.